\documentclass{article}

\usepackage{arxiv}

\usepackage{amsmath,amsfonts,bm}

\def\eqref#1{equation~\ref{#1}}

\def\1{\bm{1}}

\def\vw{{\bm{w}}}
\def\vx{{\bm{x}}}
\def\vy{{\bm{y}}}
\def\vz{{\bm{z}}}

\DeclareMathAlphabet{\mathsfit}{\encodingdefault}{\sfdefault}{m}{sl}
\SetMathAlphabet{\mathsfit}{bold}{\encodingdefault}{\sfdefault}{bx}{n}

\usepackage{url}

\usepackage{bm}
\usepackage{amsmath}
\usepackage{graphicx}
\usepackage[numbers]{natbib}
\usepackage{amsthm}
\theoremstyle{plain}
\newtheorem{theorem}{Theorem}[section]
\newtheorem{proposition}[theorem]{Proposition}
\newtheorem{lemma}[theorem]{Lemma}

\theoremstyle{definition}

\theoremstyle{remark}

\newcommand{\vphi}{\bm{\phi}}

\newcommand{\tr}{\text{tr}}

\newcommand{\off}{\text{off}}

\newcommand{\llv}{\left\lVert}
\newcommand{\rrv}{\right\rVert}

\newcommand{\Order}{\mathcal{O}}
\newcommand{\gauss}{\mathcal{N}}
\newcommand{\Real}{\mathbb{R}}

\newcommand{\mcH}{\mathcal{H}}
\newcommand{\Ebb}{\mathbb{E}}

\title{Implicit Target Shift in Online Learning: Characterization and Correction}
\author{%
  Ziyan Li\\
  Department of Physics\\ 
  Washington University in St. Louis\\
  St. Louis, USA\\
  \texttt{li.ziyan@wustl.edu} \\
  \And
  Naoki Hiratani \\
  Department of Neuroscience\\
  Washington University in St. Louis \\
  St. Louis, USA \\
  \texttt{hiratani@wustl.edu} \\
}
\date{}

\begin{document}

\maketitle

\begin{abstract}
Online learning from a stream of data is a defining feature of intelligence, yet modern machine learning systems often struggle in this setting, especially under distributional shift. To understand its basic properties, we study the relationship between online and offline learning in the context of kernel regression by deriving a closed-form expression for the function learned by online kernel regression. We reveal that online kernel regression is equivalent to offline regression with shifted, inaccurate target outputs.
Conversely, we show that by compensating for this implicit target shift in the teaching signal through target correction, online kernel-based learning can provably learn the same predictor as its offline counterpart. We derive both a closed-form expression for this target correction and an iterative form that can be applied sequentially.
Applying this framework to continual image classification tasks on domain-incremental Split CIFAR-10 and CORe50, we show that online stochastic gradient descent with iteratively corrected targets outperforms learning with the true targets. This work therefore provides a basic framework for analyzing and improving online learning in non-stationary environments from the perspective of implicit target shift.
\end{abstract}



\section{Introduction}
Humans and animals can adaptively learn new skills and knowledge from streams of sensory inputs and environmental feedback. This online form of learning in non-stationary environments is becoming increasingly important in machine learning applications, including large language model fine-tuning \citep{luo2025empirical, carta2023grounding}, robotics \citep{lesort2020continual, wolczyk2021continual}, and medical diagnosis \citep{lee2020clinical}. However, neural network-based models typically require many epochs of updates over the same static dataset, using gradients averaged over many samples. Consequently, in the presence of large distributional shifts or continual changes in the task objective, neural networks often struggle to learn in an online manner and resulting catastrophic forgetting \citep{mccloskey1989catastrophic, hadsell2020embracing, luo2025empirical}.

Substantial empirical and theoretical progress has been made in understanding learning under distribution shift \citep{zinkevich2003online, gama2004learning, besbes2015non}, particularly in the context of continual learning \citep{french1991using, kirkpatrick2017overcoming, zenke2017continual, rebuffi2017icarl,  lee2021continual}. However, much of continual learning theory considers explicitly defined sequences of tasks \citep{bennani2020generalisation, lee2021continual, doan2021theoretical, evron2022catastrophic, lin2023theory}, leaving the relationship between sample-by-sample online learning and its offline counterpart under arbitrary sample order less well understood. This distinction is especially important in online continual learning, where data are processed sequentially, often in a single pass \citep{lopez2017gradient, aljundi2019gradient, mai2022online}. Classical online learning theory provides rigorous guarantees for sequential prediction and optimization \citep{zinkevich2003online, cesa2004generalization, rakhlin2011making, dieuleveut2016nonparametric}, but such analyses typically rely on stationarity and convexity, and thus do not directly capture the learning dynamics of neural networks under distribution shift. Consequently, we still lack a mechanistic understanding of how sequential updates under a changing data distribution bias the learned predictor relative to an offline solution trained jointly on the same data, and whether this bias can be systematically characterized and corrected.


Here, we address why online continual learning is structurally difficult, and how to mitigate this difficulty, by asking: \textit{what kind of offline learning is an online learner actually performing?} We address this question primarily in the context of kernel regression \citep{scholkopf2002learning,hofmann2008kernel} and its online formulation \citep{freund1998large, cauwenberghs2000incremental, kivinen2004online, engel2004kernel,smale2006online}. This provides a theoretically grounded framework that is intrinsically linked to the training dynamics of wide deep neural networks via the Neural Tangent Kernel (NTK) \citep{jacot2018neural, lee2019wide, arora2019exact}, and also employed in data-limited machine learning applications \citep{cervantes2020comprehensive, deisenroth2013gaussian, wilson2016deep}. 

First, we rewrite online kernel regression as a closed-form offline expression. This highlights the directional, causal sample-to-sample interactions inherent in online learning, contrasting with the bi-directional inter-sample interactions of offline batch learning. Importantly, this formulation allows us to compare online and offline learning directly. We reveal that online kernel regression is mathematically equivalent to offline kernel regression trained on shifted targets, and we provide a closed-form expression for these implicit shifts.

Second, we reveal that it is possible to reverse this implicit target shift by actively correcting the teaching signals used during online learning. Specifically, we prove that there exists a target correction that allows online kernel regression to exactly achieve the performance of its offline counterpart. While this exact correction requires a priori knowledge of all future samples, we derive an iterative approximation that respects causality. We further prove that the online-offline gap under this iterative correction is bounded by the cumulative offline learning loss and converges at $\Order(1/n)$ rate in an independent and identically distributed (i.i.d.) setting. 

Finally, we apply this target correction framework heuristically to non-linear models outside the strict kernel regime. By using corrected targets based on the evolving empirical NTK as training signals for mini-batch stochastic gradient descent (SGD), we evaluate deep neural networks on the Split CIFAR-10 \citep{Krizhevsky09learningmultiple} and CORe50 \citep{lomonaco2017core50} datasets. We show that in continual learning settings, target correction significantly outperforms vanilla SGD, providing a proof of principle for a counterintuitive conclusion from our theory: online continual learning can be improved by training on corrected target labels rather than the ground-truth targets.

\section{Related Work}
The online implementation of kernel regression has been studied extensively \citep{kivinen2004online, freund1998large, cauwenberghs2000incremental, engel2004kernel, smale2006online}. Various data sub-sampling and dictionary-building methods to keep these algorithms memory-efficient have also been proposed \citep{dekel2008forgetron, richard2008online, liu2009information, dai2014scalable}. Theoretical analyses of convergence rates have demonstrated that, in the absence of temporal distribution shifts, online kernel regression is competitive with its offline counterpart under appropriate step size choices \citep{ying2008online, tarres2014online, dieuleveut2016nonparametric, wen2025optimal, zhang2025learning}. However, these works primarily focus on i.i.d. settings. In contrast, our work explicitly addresses non-stationary environments, including class-incremental learning. While online learning in non-stationary environments has been explored, particularly in online convex optimization settings \citep{zinkevich2003online, besbes2015non}, our approach provides a fine-grained sample level analysis of the sub-optimality, through the lens of implicit target shift.

The concept of an implicit target shift is conceptually adjacent to influence functions, a diagnostic tool used to estimate the impact of individual training points on model predictions \citep{hampel1974influence, koh2017understanding}. The key distinction lies in the direction of the analysis: while influence functions measure how the model's error changes when a data point is perturbed, our implicit target shift measures how the data itself is effectively perturbed by the online learning process.

Our target correction framework also finds parallels in label smoothing \citep{szegedy2016rethinking, muller2019does} and knowledge distillation \citep{hinton2015distilling}, where training targets are softened to improve generalization. Unlike global target modifications in these methods, here we provide sample-level vectorized target correction to exactly reverse the bias stemming from online learning.  

In the context of continual learning, our target correction is related to weight regularization methods \citep{kirkpatrick2017overcoming, zenke2017continual} and gradient sparsification/orthogonalization methods \citep{french1991using, serra2018overcoming, bennani2020generalisation}. However, the proposed target correction framework offers a conceptually different perspective: rather than constraining the weights or gradients directly, we modulate the target outputs to naturally align the online trajectory with the multi-task (offline) ideal. 
This is also related to Dark Experience Replay (DER) \citep{buzzega2020dark}, in which past predictions, rather than ground-truth targets, are replayed. Our framework provides a mathematical basis for understanding how non-ground-truth targets improve online learning and how to construct those targets. 


\section{Implicit Target Shift in Online Kernel Regression}
To understand the nature and limitations of online learning, we focus on the comparison between online and offline learning.
The offline learning solution provides a baseline that is independent of sample order, allowing us to isolate the effect of sample-by-sample learning from factors shared by the two settings, such as the hypothesis class and learning objective.
In the context of continual learning, joint multi-task training (the offline ideal) is widely regarded as a theoretical upper bound for sequential learning \citep{hadsell2020embracing, mai2022online, peng2023ideal}, further motivating an understanding of how online continual learning deviates from this offline ideal.
In particular, below, we investigate the following question: What is the offline learning equivalent of an online learning process?

We address this in the context of kernel regression, where the offline predictor admits a closed-form solution. By utilizing the properties of inverse triangular matrices, we reformulate the online learning rule as a modified offline learning.
This reformulation reveals three key insights:
(1) Online learning incurs a performance penalty because it relies on a directional kernel as opposed to the symmetric kernel employed in offline regression.
(2) Online kernel regression is equivalent to offline learning with a target shift, suggesting that online updates effectively transform the teaching signals.
(3) The evolution of this implicit target shift follows a simple error-based update rule: the larger the error on a new sample, the greater the label shift it induces on neighboring samples.

\subsection{Offline Expression of Online Kernel Regression}
Consider online learning from a sequence of inputs and target outputs $\{\vx_t, \vy_t\}_{t=1}^n$, where $\vx_t \in \Real^{d_x}$ and $\vy_t \in \Real^{d_y}$.
For brevity, we define the $d_x \times n$ input matrix $X_n \equiv [\vx_1, \dots, \vx_n]$ and the $d_y \times n$ target matrix $Y_n \equiv [\vy_1, \dots, \vy_n]$.
We first focus on linear regression using a feature mapping $\phi(\vx)$, where $\vy = W \phi(\vx)$. More general scenarios are discussed in Section \ref{sec_cifar10_app}.
Given a mean-squared error (MSE) loss with $L_2$ regularization:
\begin{align}
\ell (W;\vx_t, \vy_t) = \tfrac{1}{2} \lVert \vy_t - W\phi(\vx_t) \rVert^2 + \tfrac{\gamma}{2} \lVert W \rVert_F^2,
\end{align}
it is well known \citep{hofmann2008kernel} that offline batch learning yields the following predictor for an arbitrary input $\vx_*$:
\begin{align} \label{eq_batch_kernel_regression}
f_{\text{off}} (\vx_*; X_n, Y_n) = Y_n \left( \gamma I + K(X_n, X_n) \right)^{-1} k(X_n, \vx_*),
\end{align}
where $k(\vx, \vx') \equiv \phi(\vx)^T \phi(\vx')$ is the kernel function, $K(X_n, X_n)$ is the $n \times n$ Gram matrix with elements $k(\vx_i, \vx_j)$, $k(X_n, \vx_*)$ is the $n$-dimensional kernel vector, and $I$ is the identity matrix.
In contrast, an iterative sample-by-sample gradient update in the weight space follows:
\begin{align} \label{eq_online_kr_sgd}
W_{t+1} = W_t - \eta \left( (W_t \phi(\vx_t) - \vy_t) \phi(\vx_t)^T + \gamma W_t \right),
\end{align}
where $\eta$ is a fixed learning rate. Applying this update rule recursively, we obtain:
\begin{align} \label{eq_Wt_online_kr_sgd}
W_{t+1} 
&= \sum_{i=1}^t \eta \vy_i \vphi_i^T \prod_{j=i+1}^t \left[ (1-\eta \gamma) I - \eta \vphi_j \vphi_j^T \right]
+ W_1 \prod_{j=1}^t \left[ (1-\eta\gamma)I - \eta \vphi_j \vphi_j^T \right],
\end{align}
where $\vphi_i \equiv \phi(\vx_i)$ and $W_1$ is the initial weight matrix, and we used the convention $\prod_{j=a}^b A_j = A_a A_{a+1}...A_b$. Using the properties of the inverse of an upper triangular matrix, we derive a closed-form solution for the predictor resulting from this online rule (see Appendix \ref{subsec_proof_lemma_onlie_kr_closed} for the derivation).

\begin{lemma} \label{lemma_online_kr_closed}
Given a sequence of data $(X_n, Y_n)$, the online learning predictor $f_{\text{on}} (\vx_*; X_n, Y_n) \equiv W_{n+1} \phi(\vx_*)$ after $n$ update steps from $W_1=0$ follows:
\begin{align} \label{eq_eq_okr_closed_with_reg}
f_{\text{on}} (\vx_*; X_n, Y_n)
= Y_n D_n \left( \tfrac{1}{\eta} I + \tfrac{1}{1-\eta \gamma} K^U(X_n, X_n) \right)^{-1} k(X_n, \vx_*),
\end{align}
where $K^U (X_n, X_n)$ is the strictly upper-triangular portion of the Gram matrix $K(X_n, X_n)$, and $D_n$ is a diagonal matrix with $[D_n]_{ii} =  (1-\eta \gamma)^{n-i}$.
\end{lemma}
Note that, the iterative formulation of online kernel regression is well established \citep{cauwenberghs2000incremental, kivinen2004online, engel2004kernel, smale2006online}. This lemma instead derives an offline-equivalent closed-form expression for the predictor learned by online kernel regression.

Equation \ref{eq_eq_okr_closed_with_reg} implies that regularization in the online setting is primarily driven by the learning rate factor $1/\eta$, as shown previously \citep{ying2008online, barrett2020implicit}. 
Hence, for simplicity, we set $\gamma = 0$ for online learning in the following analysis. The expression then becomes:
\begin{align} \label{eq_okr_closed_zero_reg}
f_{\text{on}} (\vx_*) = Y_n \left( \tfrac{1}{\eta} I + K^U(X_n, X_n) \right)^{-1} k(X_n, \vx_*).
\end{align}
Comparing Eq. \ref{eq_okr_closed_zero_reg} with Eq. \ref{eq_batch_kernel_regression} highlights the fundamental distinction between online and offline regimes. In batch regression, the inverse term depends on the full similarity matrix $K(X_n, X_n)$. Conversely, the online predictor uses only the strictly upper-triangular components $K^U(X_n, X_n)$. This is intuitively interpretable: since an online model lacks access to future samples, it can only incorporate similarities between the current sample and its predecessors, resulting in a directional kernel. In other words, for sample $x_t$, $k(x_t, x_{t'})$ is available for $t'\leq t$ (past samples), but not for $t'>t$ (future samples). This lack of bidirectional information flow provides a clear mathematical intuition for why online learning is inherently more constrained.

Finally, as shown in Appendix \ref{subsec_minibatch}, this framework extends to mini-batch learning with size $b$:
\begin{align} \label{eq_minibatch_main}
f_b (\vx_*) = Y_n \left( \tfrac{1}{\eta} I + K^{bU} (X_n, X_n) \right)^{-1} k(X_n, \vx_*),
\end{align}
where $K^{bU}$ is a strictly upper-triangular matrix where the $b \times b$ block-diagonal components are also zeroed out. This reflects the fact that the model cannot leverage correlations between samples within the same mini-batch during a single update step.

\subsection{Implicit Target Shift}
The structural similarity between the online kernel regression expression (Eq. \ref{eq_okr_closed_zero_reg}) and its offline counterpart (Eq. \ref{eq_batch_kernel_regression}) allows us to introduce the concept of an implicit target shift. We define the effective targets $Y^e_n$ as follows:
\begin{align} \label{eq_def_target_shift_kr}
  Y^e_n 
  \equiv Y_n \left( \tfrac{1}{\eta} I + K^U(X_n, X_n)\right)^{-1} 
  \left( \gamma I + K(X_n, X_n) \right).
\end{align}
Using this definition, we establish the following equivalence:
\begin{theorem} \label{target_shift_theorem}
Assuming $\eta, \gamma > 0$, the predictor $f_{\text{on}}$ obtained through online gradient descent $W_{t+1} = W_t - \eta (W_t \phi(\vx_t) - \vy_t) \phi(\vx_t)^T$ with $W_1=0$ on data $(X_n, Y_n)$ is identical to the offline kernel regression predictor $f_{\text{off}}$ trained on the data with modified target $(X_n, Y^e_n)$. That is:
\begin{align}
  f_{\text{on}}(\vx; X_n, Y_n) = f_{\text{off}} (\vx; X_n, Y^e_n).
\end{align}
\end{theorem}
\paragraph{Proof:}
By substituting the definition of $Y^e_n$ from Eq. \ref{eq_def_target_shift_kr} into the online predictor formula (Eq. \ref{eq_okr_closed_zero_reg}):
\begin{align}
f_{\text{on}} (\vx_*; X_n, Y_n) 
    &= Y_n \left( \tfrac{1}{\eta} I + K^U(X_n, X_n)\right)^{-1} k(X_n, \vx_*)
    \nonumber \\
    &= Y^e_n \left( \gamma I + K(X_n, X_n)\right)^{-1} k(X_n, \vx_*) 
    \nonumber \\
    &= f_{\off} (\vx_*; X_n, Y^e_n). 
\end{align}

This theorem demonstrates that the sub-optimality of online kernel learning can be fully captured by a transformation of the target outputs from $Y_n$ to $Y_n^e$.
Note that the online predictor $f_{\text{on}}$ in Theorem \ref{target_shift_theorem} is defined without an explicit regularizer (Eq. \ref{eq_okr_closed_zero_reg}). The parameter $\gamma$ in Eq. \ref{eq_def_target_shift_kr} originates solely from the offline predictor $f_{\text{off}}$. Consequently, the effective targets $Y^e_n$ depend directly on the regularization strength $\gamma$ of the specific offline baseline that the online model effectively mimics.

Using the effective target $Y^e_n$, the online-offline gap is written as $f_{\text{on}} (\vx;X_n,Y_n) - f_{\off} (\vx;X_n,Y_n) = (Y^e_n - Y_n) (\gamma I + K(X_n, X_n))^{-1} k(X_n, \vx)$. Thus, the discrepancy is upper bounded by the target shift: 
$\llv f_{\text{on}} (\,\cdot\,;X_n,Y_n) - f_{\off} (\,\cdot\,;X_n,Y_n) \rrv_{\mathcal{H}} \leq \tfrac{1}{2\sqrt{\gamma}} \llv Y^e_n - Y_n \rrv_F$, where $\mathcal{H}$ is the Reproducing Kernel Hilbert Space (RKHS) induced by kernel $k$.

Figure \ref{fig:kernel_reg}A illustrates this phenomenon in a one-dimensional regression task using data generated from a Gaussian Process (GP; see Appendix \ref{subsec_toy_model_det}). While the offline predictor (blue line) successfully recovers the true underlying function (gray line) from noisy samples (blue points), the online predictor (orange line) exhibits significant deviations. The effective target shift provides a lens for interpreting these errors: the orange crosses represent the effective targets $y^e_i$. In regions where $f_{\text{on}}(x)$ diverges from the batch solution, the effective targets have been shifted away from the true labels. Because the online predictor is mathematically equivalent to a batch regressor trained on these "distorted" targets, the resulting function inherits these deviations.

\begin{figure}[t!]
  \begin{center}
\centerline{\includegraphics[width=1.0\columnwidth]{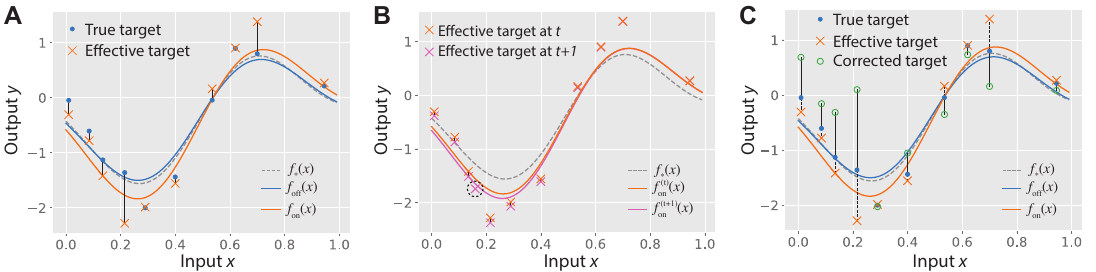}}
    \caption{ Illustration of implicit target shift and its correction. 
    \textbf{(A)} Example of implicit target shift. Blue and orange points correspond to the true target $y_i$ and effective target $y_i^e$, respectively.
    \textbf{(B)} Shift in the effective target from one-step update (with the sample marked by a dashed black circle). Pink points on the left side shift downward from orange points, while they overlap on the right side.  
    \textbf{(C)} Target correction (green circles) for the implicit target shift (orange x-marks) illustrated in panel A. 
    See main text and Appendix~\ref{subsec_toy_model_det} for details. 
    }
    \label{fig:kernel_reg}
  \vspace{-10pt}
  \end{center}
\end{figure}

\subsection{One-step Update of Implicit Target Shifts}
To understand the mechanism by which these distortions emerge, we analyze the dynamics of the target shift. Let $\vy^e_i (n)$ denote the effective target of the $i$-th sample after $n$ updates. Upon receiving the $(n+1)$-th sample, the effective targets evolve as follows (see Appendix \ref{subsec_target_shift_dynamics_kr} for the derivation):
\begin{align} \label{eq_effective_target_change}
  \vy^e_i (n+1)
  &= \begin{cases}
  \vy^e_i (n) - \eta \bm{e}_{n+1} k(\vx_{n+1}, \vx_i)
  & i \leq n \\
  \vy_{n+1} - (\eta [\gamma + k_{n+1}]-1) \bm{e}_{n+1} 
  & i = n+1,
\end{cases}
\end{align}
where $k_{n+1} = k(\vx_{n+1}, \vx_{n+1})$ and $\bm{e}_{n+1} \equiv f_{\text{on}} (\vx_{n+1}; X_n, Y_n) - \vy_{n+1}$ is the prediction error on $\vx_{n+1}$ from data $(X_n,Y_n)$. 
For previous samples ($i \leq n$), the shift is proportional to the prediction error $\bm{e}_{n+1}$ and the kernel similarity $k(\vx_{n+1}, \vx_i)$. This indicates that a poor prediction on a new point biases the effective targets of its neighbors in the feature space.

Figure \ref{fig:kernel_reg}B visualizes this one-step shift. When a new sample (marked by a dashed circle) is added to the dataset from Fig. \ref{fig:kernel_reg}A, it induces a downward shift in the effective targets of its neighbors (left side), while leaving distant points unaffected (right side). 
This framework thus allows us to precisely localize where and how online predictors deviate from the offline ideal, providing a potential diagnostic tool for online learning performance.

\section{Target Correction}
The concept of implicit target shift introduced in the previous section helps identify which samples accumulate excess error during online learning, but it also raises a question of how to actively correct this shift.
In this section, we investigate the correction of the implicit target shift. First, we demonstrate the existence of a set of corrected targets that allows an online learner to exactly recover the offline regression predictor, providing its closed-form expression. We then derive an iterative formulation of this target correction that stabilizes the learning dynamics and restores practical causality.

By inverting the relationship defining the implicit target shift (using Eqs. \ref{eq_batch_kernel_regression} and \ref{eq_okr_closed_zero_reg}), we define the corrected targets $Y^c_n$ as:
\begin{align} \label{eq_def_target_correction_kr}
Y^c_n \equiv Y_n \left( \gamma I + K(X_n, X_n) \right)^{-1} \left( \tfrac{1}{\eta} I + K^U(X_n, X_n)\right).
\end{align}
\begin{theorem} \label{target_correction_theorem}
Assuming $\eta, \gamma > 0$, the predictor obtained via online kernel-based learning $W_{t+1} = W_t - \eta (W_t \phi(\vx_t) - \vy_t^c ) \phi(\vx_t)^T$ starting from $W_1=0$ on the data $(X_n, Y_n^c)$ is mathematically equivalent to the offline learning predictor trained on the original data $(X_n, Y_n)$. Namely:
\begin{align}
f_{\text{on}}(\vx; X_n, Y_n^c) = f_{\text{off}} (\vx; X_n, Y_n).
\end{align}
\end{theorem}
The proof follows straightforwardly from a parallel argument to Theorem \ref{target_shift_theorem}. This result implies that if the true targets are known a priori, one can construct a set of "distorted" targets $Y^c_n$ with which an online learner can provably match offline performance. Counterintuitively, this means that an online model can achieve better performance by training on intentionally altered targets rather than the ground-truth targets.

Figure \ref{fig:kernel_reg}C illustrates this correction for the 1D regression task previously shown in Figure \ref{fig:kernel_reg}A. In the left-hand region, where the effective targets (orange crosses) are erroneously shifted downward by the online process, the corrected targets (green open circles) introduce an upward bias to cancel the induced error. Conversely, on the right side, the corrected targets impose a downward correction. Theorem \ref{target_correction_theorem} ensures that these counter-biases exactly align the online learning trajectory with the offline baseline, meaning that the online predictor constructed with green corrected points exactly matches the blue curve, $f_{\off} (x)$. 
As with target shift, we can also analyze the step-wise evolution of these corrections to understand their directionality (see Appendix \ref{subsec_target_correction_dynamics_kr}). 

\subsection{Iterative Target Correction}
While the exact target correction guarantees offline performance, its construction (Eq. \ref{eq_def_target_correction_kr}) requires full access to the entire dataset a priori, breaking the causal requirement of online learning. Furthermore, enforcing this exact correction can result in highly unstable learning dynamics (as will be demonstrated in Fig. \ref{fig:conv_ntk}C and D). To address these limitations, we derive an iterative formulation.

We first cast the corrected targets (Eq. \ref{eq_def_target_correction_kr}) as the solution to a minimization problem. Given an arbitrary set of targets $Z \in \Real^{d_y \times n}$, the online prediction is 
$f_{\text{on}} (\vx;X,Z) = Z (\tfrac{1}{\eta} I + K^U (X,X))^{-1} k(X, \vx)$. 
Thus, the discrepancy between the online and offline predictors is minimized when:
\begin{align}
Z = \arg \min_Z \lVert Z (\tfrac{1}{\eta}I + K^U(X,X))^{-1} k(X, \cdot) - Y (\gamma I + K(X,X))^{-1} k(X, \cdot) \rVert_{\mathcal{H}}^2.
\end{align}
Solving this recovers the original target correction (Eq. \ref{eq_def_target_correction_kr}; see Appendix \ref{subsec_lossmin_tc}).

To adapt this to an iterative, causal setting, we optimize targets in sequential chunks. Consider the optimization of novel targets $Z_{\text{new}} = \{ \vz_{N-b+1}, \dots, \vz_{N}\}$ for incoming inputs $X_{\text{new}} = \{ \vx_{N-b+1}, \dots, \vx_N \}$, conditioned on past inputs $X_{\text{past}} = \{ \vx_{1}, \dots, \vx_{N-b} \}$ and their frozen corrected targets $Z_{\text{past}} = \{ \vz_1, \dots, \vz_{N-b}\}$. Here, $b$ is the look-ahead horizon (or mini-batch size) over which targets are simultaneously optimized. While we still assume a sample-by-sample online update for the actual weight learning, the targets themselves are generated iteratively in blocks of size $b$.

Denoting the union of inputs as $X_{\text{tot}} = X_{\text{past}} \cup X_{\text{new}}$, the block-wise loss function becomes:
\begin{align} \label{eq_iter_correction_loss}
\ell(Z_{\text{new}} ; Z_{\text{past}}, X_{\text{tot}}) 
&= \frac{1}{2} \llv [Z_{\text{past}}, Z_{\text{new}}] ( \tfrac{1}{\eta}I + K^U_{\text{tot}} )^{-1} k(X_{\text{tot}} , \cdot) - Y_{\text{tot}} (\gamma I + K_{\text{tot}})^{-1} k(X_{\text{tot}}, \cdot) \rrv_{\mathcal{H}}^2
\nonumber \\
&\quad + \frac{\gamma_o}{2} \llv [Z_{\text{past}}, Z_{\text{new}}] (\tfrac{1}{\eta} I + K^U_{\text{tot}} )^{-1} \rrv_F^2,
\end{align}
where $Y_{\text{tot}} (= Y_{\text{past}} \cup Y_{\text{new}})$ represents the true labels for the accumulated data, and the second term is a Tikhonov regularizer ($\gamma_o > 0$) introduced to enforce stability. Minimizing this objective yields the following update rule:
\begin{proposition} \label{prop_iter_correction}
Minimization of Eq. \ref{eq_iter_correction_loss} yields a unique solution for $Z_{\text{new}}$ given $\gamma, \eta, \gamma_o > 0$:
\begin{align} \label{eq_citer_Znew}
 Z_{\text{new}} 
 = Y_{\text{new}} 
 + (Y_{\text{new}} - f_{\text{on}}(X_{\text{new}}; X_{\text{past}},  Z_{\text{past}}))C_{\text{on}} 
 + (f_{\text{off}} (X_{\text{new}}; X_{\text{past}}, Y_{\text{past}}) - Y_{\text{new}}) C_{\text{off}},
\end{align}
where the coefficient matrices are defined as:
\begin{subequations} \label{eq_def_Con_Coff}
\begin{align}
  C_{\text{on}} 
  &\equiv \left((\gamma_o I + K_{nn})^{-1} (\tfrac{1}{\eta} I + K^U_{nn}) - I \right), \\
  C_{\text{off}} 
  &\equiv \gamma \left( [\gamma I + K_{nn}] - K_{np} [\gamma I + K_{pp}]^{-1} K_{pn} \right)^{-1} (\gamma_o I + K_{nn})^{-1} (\tfrac{1}{\eta} I + K^U_{nn}),
\end{align}
\end{subequations}
with $K_{pp} = K(X_{\text{past}}, X_{\text{past}})$ and $K_{np} = K(X_{\text{new}}, X_{\text{past}})$.
\end{proposition}
See Appendix \ref{subsec_iter_c_proof} for the full derivation.
This formulation disentangles two mechanisms of correction. The first term, weighted by $C_{\text{on}}$, acts as a sample-wise learning rate modulation driven by the online prediction error. This can be seen by rearranging Eq. \ref{eq_citer_Znew} to highlight the modified residual:
\begin{align}
Z_{\text{new}} - f_{\text{on}}= (Y_{\text{new}} - f_{\text{on}}) (I + C_{\text{on}}) + (f_{\text{off}} - Y_{\text{new}} ) C_{\text{off}},
\end{align}
where $f_{\text{on}}$ and $f_{\text{off}}$ represent the online and offline predictions on $X_{\text{new}}$, respectively.
The second term, weighted by $C_{\text{off}}$, injects the offline estimation structure into the targets. Crucially, the offline predictor $f_{\text{off}}$ in this formulation relies only on $X_{\text{past}}$, $X_{\text{new}}$, and $Y_{\text{past}}$, circumventing the causality bottleneck, and allowing for real-time target shifts without requiring future labels.

Although the iterative target correction no longer guarantee the exact recovery of the offline predictor, the following two results hold on the discrepancy between online and offline predictors under this iterative correction. 
\begin{proposition} \label{prop_iter_tc_bound_noniid}
Consider iterative target correction with data $X_{1:n} = [X_1, ...,X_n]$ and $Y_{1:n} = [Y_1, ..., Y_n]$. 
Then, under $\gamma > 0$ and $\gamma_o > 0$, the online-offline gap is bounded by:
\begin{align}
    \llv f_{on}(\,\cdot\,; X_{1:n}, Z_{1:n}) - f_{\off} (\,\cdot\,; X_{1:n}, Y_{1:n}) \rrv_{\mathcal{H}}
    \leq \frac{1}{\sqrt{\gamma}} \sum_{i=1}^n \llv Y_i - f_{\off} (X_i; X_{1:i-1}, Y_{1:i-1})\rrv_{F}.
\end{align}
\end{proposition}
The proof is provided in Appendix \ref{subsec_noniid_bound}. Thus, even for an arbitrary data sequence, the online--offline gap under iterative target correction is bounded by the cumulative prediction error of predictors trained on all preceding samples. By contrast, for naive online learning, the corresponding gap depends on the minimum singular value of the directional kernel, $\sigma^{-1}_{\min} \left( \gamma I + K^U(X_{1:n}, X_{1:n}) \right)$, which need not be lower bounded by $1/\gamma$ due to non-normality (see Eq. \ref{eq_naive_online_gap_bound} in Appendix). 

Moreover, under an i.i.d. setting with additional constraints on data and kernel (see Assumptions \ref{subsec_iid_bound}), the online-offline gap converges at $\Order(1/n)$ rate, supporting the soundness of the algorithm:

\noindent \textbf{Proposition \ref{prop_iter_tc_bound_iid} (informal):} Under Assumptions \ref{subsec_iid_bound}, the norm of online-offline gap follows, 
\begin{align}
    \Ebb [ \llv f_{on}(\,\cdot\,; X_{1:n}, Z_{1:n}) - f_{\off} (\,\cdot\,; X_{1:n}, Y_{1:n}) \rrv_{\mcH}^2] 
    \leq C \left( \frac{1}{n} + \frac{b}{2} (1 - \alpha)^{n/2} \right)
\end{align}
where $C$ and $\alpha\, (0 <\alpha < 1)$ are constants that depends on the kernel, probability distribution, and the target function. The full statement and the proof are provided in Appendix \ref{subsec_iid_bound}.

\subsection{Example: Conv-NTK Regression on MNIST}
\begin{figure*}[t]
  \begin{center}
    \centerline{\includegraphics[width=1.0\linewidth]{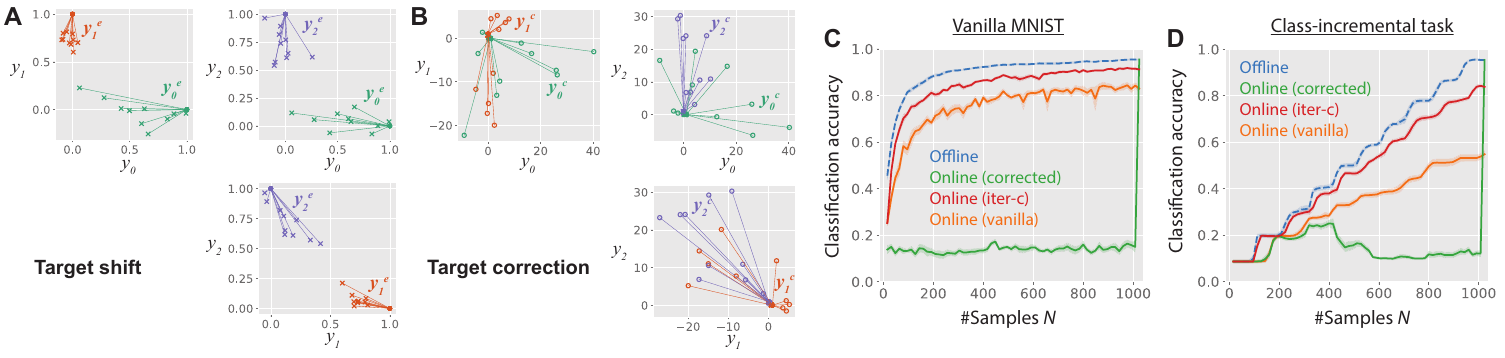}}
    \caption{
    Effective target shifts and their correction in NTK regression applied to MNIST. 
    \textbf{A, B)} Visualization of effective target shifts (A) and their correction (B). Three panels depict the first three dimensions of the output space corresponding to labels 0, 1, and 2. For instance, green lines in the top-left panel in A depict how the one-hot vector of label 0, $\vy_0 = [1,0,0,...]$, exhibits effective shifts under online NTK regression (10 lines represent 10 training samples belonging to label 0). 
    \textbf{C)} Learning curves of offline (blue) and online kernel regression with the true (orange), corrected (green), and iteratively-corrected (red) targets. 
    \textbf{D)} The same as C, but in a class-incremental setting. In panels C and D, accuracy was measured by test accuracy, and shaded areas are the standard error of the mean (SEM) over 10 random seeds.   
    See Appendix \ref{subsec_ntk_reg_det} for details. 
    }
    \label{fig:conv_ntk}
  \vspace{-10pt}
  \end{center}
\end{figure*}

Figure \ref{fig:conv_ntk} demonstrates iterative target correction on the MNIST image recognition task, utilizing the NTK of a shallow, infinite-width convolutional neural network \citep{neuraltangents2020}. Following standard NTK practice, we framed the classification as a regression problem via mean-squared error \citep{jacot2018neural, arora2019exact,liu2020linearity}.

Visualizing the implicit target shift in the 3D subspace spanned by the digits '0', '1', and '2', we observed that online learning consistently causes a shrinkage of the target vectors toward the origin (Fig. \ref{fig:conv_ntk}A; filled circles are true one-hot targets, x-marks are effectively shifted targets after 1024 online updates). 
The target correction compensates for this shrinkage by proactively pushing the target vectors apart (Fig. \ref{fig:conv_ntk}B). 
Interestingly, the correction flips the signs of certain sample targets entirely, highlighting the non-trivial geometry of optimal label correction.

When evaluating test accuracy (Fig. \ref{fig:conv_ntk}C), vanilla online learning underperformed the offline NTK baseline (orange vs. blue lines), as expected. Training sequentially with the exact corrected labels (Eq. \ref{eq_def_target_correction_kr}) eventually matches the offline NTK (green line), but suffered from instability during intermediate learning phases. In contrast, the iterative target correction ("iter-c"; red line) maintained near-offline stability and performance throughout the entire learning trajectory. Note that both offline $\gamma$ and online $\eta$ were optimized for peak final accuracy (see Fig. \ref{fig:conv_ntk_supp} in Appendix), confirming that the iterative correction's advantage stems from the sample-wise modulation via $C_{\text{on}}$ and $C_{\text{off}}$ rather than global hyperparameter tuning.

Finally, we evaluated target correction in a highly non-stationary class-incremental learning setting, presenting all samples belonging to label '0' (approx. 100 samples) before moving to label '1', and so on. Under this adversarial trial order, vanilla SGD struggled to learn the task effectively due to forgetting. However, learning under iterative target correction performed close to offline levels (Fig. \ref{fig:conv_ntk}D), demonstrating its resistance to temporal distribution shifts.

\section{Application of Target Correction to Non-linear Models} \label{sec_cifar10_app}
  
Building on the offline-equivalent expression of online kernel regression, we introduced the concept of implicit target shift and showed that iterative target correction can improve online learning performance. We now extend this framework beyond the linear kernel regime to online mini-batch SGD in non-linear models.
Because features evolve during deep network training \citep{fort2020deep, geiger2020disentangling, radhakrishnan2024mechanism}, a closed-form, static target correction is no longer theoretically guaranteed. Nevertheless, we hypothesized that applying the correction heuristically using an evolving empirical kernel approximation could still improve online learning, particularly in continual learning settings.

To adapt target correction to non-linear networks, we computed the empirical NTK, $\nabla_{\vw} f^T \nabla_{\vw} f$, directly from the network and substituted it into Eq. \ref{eq_citer_Znew} to calculate corrected targets. To account for feature learning while limiting computational cost, we updated the empirical NTK only at the beginning of each new task and, for multi-epoch learning, at the beginning of each within-task epoch. We used mini-batch SGD throughout (see Appendix \ref{subsec_nonlinear_cl_det} for full implementation details).

\begin{figure*}[t]
  \begin{center}
\centerline{\includegraphics[width=1.0\linewidth]{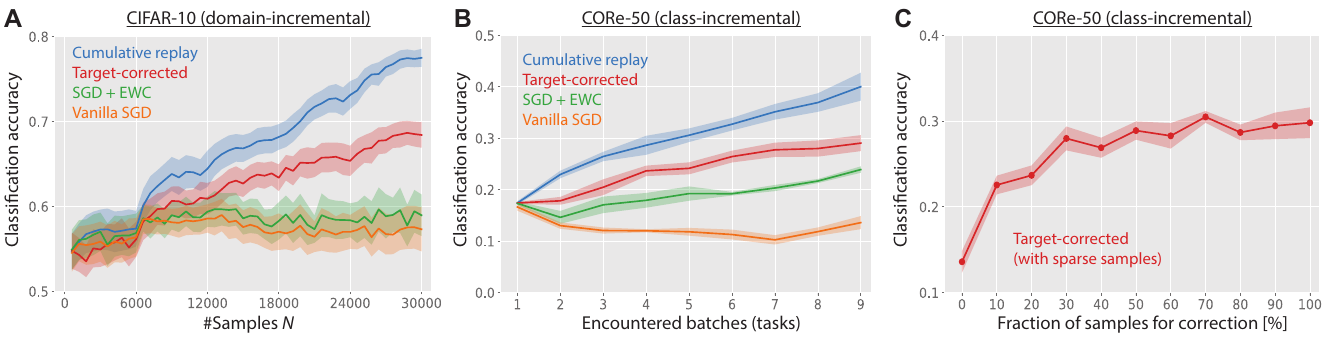}}
    \caption{
    Application of iterative target correction to SGD training of nonlinear neural networks. 
    \textbf{(A)} Domain-incremental setting where 10 categories from CIFAR-10 are split into five binary tasks. We used a single head for all five tasks to make the task challenging. 
    \textbf{(B)} CORe-50 dataset, a class-incremental continual learning task with 50 categories in total, split into 9 tasks. 
    \textbf{(C)} Final accuracy on CORe50 as a function of the fraction of samples used for target correction. See Appendix \ref{subsec_nonlinear_cl_det} for implementation details. }
  \label{fig:conv_sgd}
  \vspace{-10pt}
  \end{center}
\end{figure*}

We first evaluated iterative target correction in a domain-incremental continual learning setting based on CIFAR-10 \citep{Krizhevsky09learningmultiple}, where the dataset was split into five binary classification tasks solved sequentially with a single head. Online learning with iterative target correction significantly outperformed both vanilla SGD and the Elastic Weight Consolidation (EWC) baseline \citep{kirkpatrick2017overcoming} (Fig. \ref{fig:conv_sgd}A; red vs. orange and green lines), although its performance was naturally bounded below that of cumulative replay, where the model was trained jointly on all previously seen data (red vs. blue). Learning rates and regularization strengths were optimized independently for each condition, supporting the robustness of this advantage (see Fig. \ref{fig:conv_sgd_supp_hyper} in Appendix). Target correction with a fixed NTK also outperformed vanilla SGD and EWC in this setting (Fig. \ref{fig:conv_sgd_supp_var}A).

To test whether this observation extends beyond the artificial CIFAR-10 task decomposition, we next evaluated the method on the CORe50 continual learning benchmark \citep{lomonaco2017core50} (Fig. \ref{fig:conv_sgd}B). While the absolute performance was limited by our simplified architectural setup, and the performance degradation was observed in vanilla SGD (orange line) due to task difficulty as noted in the original CORe50 benchmark \citep{lomonaco2017core50}, the relative advantage of the corrected targets remained consistent. Additionally, target correction with sparse sampling of past data ($30\%$) was sufficient (Fig. \ref{fig:conv_sgd}C), mitigating the $\Order(n^3)$ computational cost of updating and inverting the empirical NTK over past data. 

The performance gain of target correction was not explained by the soft-labeling effect. Moreover, the combination of target correction with DER improved replay performance over vanilla DER (Table \ref{tab:replay_comparison} in Appendix), supporting the benefit of sample-level vectorized target correction beyond vanilla online learning. Single-pass learning with sparse target correction did not outperform replay-based methods that use the same samples directly for replay rather than for target correction (Table \ref{tab:replay_comparison}), but required only half as many network weight updates because it does not perform replay.

Together, these results indicate that iterative target correction can improve online learning in nonlinear networks across multiple continual learning settings. We emphasize, however, that in its current formulation, iterative target correction is not intended as a practical, drop-in replacement for state-of-the-art continual learning algorithms. Rather, these experiments serve as a proof of principle for our theoretical insight from kernel regression: an online model can achieve better predictive performance when trained on corrected targets than when trained directly on the ground-truth targets.
\vspace{-10pt}

\section{Discussion}
Overall, this work presents a fundamental framework to analyze the discrepancy between online and offline learning in feature-based learning settings, providing two core theorems on online-to-offline conversion. Furthermore, it proposes target correction methods for efficient online learning with theoretical bounds on the online–offline gap, which we applied to continual image classification tasks as a proof of principle.

While Theorem \ref{target_correction_theorem} guarantees that the online learner can exactly match the offline predictor, this naturally implies that generalization performance remains bottlenecked by the quality of the chosen offline regularizer, $\gamma$. More broadly, strictly enforcing exact offline equivalence removes the implicit regularization of mini-batch SGD noise, which helps networks escape sharp minima and achieve better generalization \citep{keskar2016large, zhu2018anisotropic}. 

Nevertheless, beyond conceptual interest, the target correction framework may be beneficial, particularly when modifying the network weights is the primary bottleneck and replay-based methods are therefore costly. For example, in an edge-computing setting \citep{chen2019deep}, corrected targets could be computed on a central server with greater computational resources, while the edge device performs only an online single-pass SGD update using the corrected target supplied by the server. Similar advantage is expected in a machine teaching setting where the teacher needs to optimize its feedback to the student \citep{zhu2018overview}, and also in \textit{in situ} training of optical neural networks where weight update requires a physical reconfiguration of the system \citep{buckley2023photonic}.

The exact formulation of effective target shift and correction also holds promising potential for applications in learning process diagnosis \citep{koh2017understanding}, knowledge distillation \citep{hinton2015distilling}, and reward shaping \citep{dorigo1994robot, ng1999policy}. More generally, our work establishes a new mathematical lens through which to analyze online learning, opening up a promising line of both theoretical and empirical investigation.

\section*{Acknowledgment}
This work was partially supported by the McDonnell Center for Systems Neuroscience.


{\small
\bibliography{refs}
}

\newpage
\appendix
\section{Proofs and derivations}
\subsection{Mathematical notation}
Throughout the manuscript, we used lower-case italic letters for scalar variables, lower-case bold-italic letters for vectors, and upper-case letters for matrices.
Vectors are defined as column vectors; thus $\bm{a}^T \bm{b}$ and $\bm{a} \bm{b}^T$ represent the inner and outer products of two vectors $\bm{a}$ and $\bm{b}$, respectively. 
For a vector $\bm{v}$, $[\bm{v}]_i$ represents the $i$-th element of the vector, and for a matrix $M$, $[M]_{ij}$ represents the $(i,j)$-th element of the matrix.

\subsection{Proof of Lemma \ref{lemma_online_kr_closed}} \label{subsec_proof_lemma_onlie_kr_closed}
From Eq. \ref{eq_Wt_online_kr_sgd}, the function obtained by online learning from data $(X_n, Y_n)$ obeys
\begin{align}
    f_{on}(\vx_*; X_n, Y_n)
    &= W_{n+1} \phi(\vx_*) 
    \nonumber \\
    &= \sum_{i=1}^n \eta \vy_i \vphi_i^T \prod_{j=i+1}^n \left[ (1-\eta \gamma) I - \eta \vphi_j \vphi_j^T \right] \vphi_*
   + W_1 \prod_{j=1}^n \left[ (1-\eta\gamma)I - \eta \vphi_j \vphi_j^T \right] \vphi_*,
\end{align}
where $\vphi_* \equiv \phi(\vx_*)$. 
The first term is rewritten as
\begin{align}
 & \vphi_i^T \prod_{j=i+1}^n \left[ (1-\eta \gamma) I - \eta \vphi_j \vphi_j^T \right] \vphi_*
 \nonumber \\
 &= (1-\eta \gamma)^{n-i} k(\vx_i, \vx_*) 
 + \sum_{\ell=1}^{n-i} (1 - \eta\gamma)^{n-i-\ell} (-\eta)^{\ell} \sum_{i+1 \leq j_1 < ... < j_\ell \leq n} k(\vx_i, \vx_{j_1}) k(\vx_{j_1}, \vx_{j_2}) ... k(\vx_{j_\ell}, \vx_*)
 \nonumber \\
 &= (1-\eta \gamma)^{n-i} k(\vx_i, \vx_*) 
 + \sum_{\ell=1}^{n-i} (1 - \eta\gamma)^{n-i-\ell} (-\eta)^{\ell} \left[ \left( K^U (X_n, X_n) \right)^\ell k (X_n, \vx_*) \right]_i,
\end{align}
where $[\bm{v}]_i$ represents the $i$-th element of vector $\bm{v}$. 
In the last line, we define an $n \times n$ matrix $K^U(X_n, X_n)$ as the strictly upper-triangular component of the kernel matrix:
\begin{align} \label{def_KU}
\left[ K^U (X_n, X_n) \right]_{ij} = \begin{cases}
k(\vx_i, \vx_j) & (i < j) \\
0 & (i \geq j).
\end{cases}
\end{align}
The last line follows because $K^U$ satisfies
\begin{align}
    &\left[ \left( K^U (X_n, X_n) \right)^\ell k (X_n, \vx_*) \right]_i
    \nonumber \\
    &= \sum_{j_1=1}^n \sum_{j_2=1}^n ... \sum_{j_\ell=1}^n
    K^U_{i j_1} K^U_{j_1 j_2} ... K^U_{j_{\ell-1}j_\ell} k(\vx_{j_\ell}, \vx_*) 
    \nonumber \\
    &= \sum_{j_1=i+1}^n \sum_{j_2=j_1+1}^n ... \sum_{j_\ell=j_{\ell-1}+1}^n
    k(\vx_i, \vx_{j_1}) k(\vx_{j_1}, \vx_{j_2}) ... k(\vx_{j_{\ell-1}}, \vx_{j_\ell}) k(\vx_{j_\ell}, \vx_*) 
    \nonumber \\
    &= \sum_{i+1 \leq j_1 < j_2 < ... < j_\ell \leq n}
    k(\vx_i, \vx_{j_1}) k(\vx_{j_1}, \vx_{j_2}) ... k(\vx_{j_{\ell-1}}, \vx_{j_\ell}) k(\vx_{j_\ell}, \vx_*). 
\end{align}
Therefore, the coefficient of $\vy_i$ becomes
\begin{align}
    &\vphi_i^T \prod_{j=i+1}^n \left[ (1-\eta \gamma) I - \eta \vphi_j \vphi_j^T \right] \vphi_*
    \nonumber \\
    &= (1 - \eta \gamma)^{n-i} \left( k(\vx_i, \vx_*) 
    + \sum_{\ell=1}^{n-i} \left( \tfrac{-\eta}{1 - \eta\gamma} \right)^{\ell} \left[ \left( K^U (X_n, X_n) \right)^\ell k (X_n, \vx_*) \right]_i \right)
    \nonumber \\
    &= (1 - \eta \gamma)^{n-i} \sum_{\ell=0}^{n-i} \left( \tfrac{-\eta}{1 - \eta\gamma} \right)^{\ell} \left[ \left( K^U (X_n, X_n) \right)^\ell k (X_n, \vx_*) \right]_i
    \nonumber \\
    &= (1 - \eta \gamma)^{n-i} \sum_{\ell=0}^n  \left[ \left( \tfrac{-\eta}{1 - \eta\gamma} K^U (X_n, X_n) \right)^\ell k (X_n, \vx_*) \right]_i
    \nonumber \\
    &= (1 - \eta \gamma)^{n-i} \left[ \left(
    I + \tfrac{\eta}{1 - \eta\gamma} K^U (X_n, X_n) \right)^{-1} k(X_n, \vx_*)
    \right]_i.
\end{align}
In the third line, we used that, $\left[ \left(  K^U (X_n, X_n) \right)^\ell k (X_n, \vx_*) \right]_i$ is zero for $\ell > n-i$. 
The last line follows because, for an arbitrary square matrix $A$,
\begin{align}
    (I + A) \sum_{\ell=0}^n (-A)^{\ell}
    = I + (-1)^n A^{n+1},
\end{align}
and a size $n$ strictly upper triangular matrix $K^U$ satisfies $(K^U)^{n+1} = 0$. 
Because eigenvalues of a triangular matrix are its diagonal components, $I + \tfrac{\eta}{1 - \eta\gamma} K^U (X_n, X_n)$ is invertible (unless $1 = \eta\gamma$).

Let us denote a size $n$ diagonal matrix $D_n$ by 
\begin{align}
    [D_n]_{ij} = \begin{cases}
        (1 - \eta \gamma)^{n-i} & (i=j) \\
        0 & (i \neq j), 
    \end{cases}
\end{align}
then under $W_1 = 0$, $f_{on}$ follows
\begin{align}
    f_{on}(\vx_*; X_n, Y_n)
    &= \eta \sum_{i=1}^n \vy_i (1 - \eta \gamma)^{n-i} \left[ \left(
    I + \tfrac{\eta}{1 - \eta\gamma} K^U (X_n, X_n) \right)^{-1} k(X_n, \vx_*)
    \right]_i
    \nonumber \\
    &= \eta Y_n D_n \left(
    I + \tfrac{\eta}{1 - \eta\gamma} K^U (X_n, X_n) \right)^{-1} k(X_n, \vx_*)
    \nonumber \\
    &= Y_n D_n \left(
    \tfrac{1}{\eta} I + \tfrac{1}{1 - \eta\gamma} K^U (X_n, X_n) \right)^{-1} k(X_n, \vx_*).
\end{align}

\subsection{Offline expression of mini-batch online kernel regression} \label{subsec_minibatch}
Our results for sample-by-sample learning in Lemma \ref{lemma_online_kr_closed} can be extended to a mini-batch online learning setting straightforwardly. 
Let us denote the mini-batch size as $b$, and the sample set for $i$-th iteration by $(X_i, Y_i)$ where $X_i \in \Real^{d_x \times b}$ and $Y_i \in \Real^{d_y \times b}$. For brevity, we write $\phi(X_i) = \Phi_i$, $X = [X_1, ..., X_n]$, and $Y = [Y_1, ..., Y_n]$. 
Under an online gradient learning without an L2 regularizer, the weight update follows:
\begin{align} \label{def_minibatch_sgd_kr}
    W_{t+1} = W_t - \eta (W_t \Phi_t - Y_t) \Phi_t^T.
\end{align}
By applying it recursively from $W_1 = 0$, we have
\begin{align}
    W_{t+1} = \sum_{i=1}^t \eta Y_i \Phi_i^T \prod_{j=i+1}^t \left( I - \eta \Phi_j \Phi_j^T \right), 
\end{align}
and thus the predictor $f_b^{(n)} (X_*)$ after $n$-th updates obeys
\begin{align}
    f_b^{(n)} (X_*) = \sum_{i=1}^n \eta Y_i \Phi_i^T \prod_{j=i+1}^n \left( I - \eta \Phi_j \Phi_j^T \right) \Phi_*,
\end{align}
We introduce a block-wise upper triangular matrix $K^{bU} \in \Real^{nb \times nb}$ by
\begin{align} \label{def_KbU}
    \left[ K^{bU} (X, X) \right]_{block (i,j)} =
    \begin{cases}
        K (X_i, X_j) & (i < j) \\
        0 & (i \geq j),
    \end{cases}
\end{align}
where $\left[ K^{bU} \right]_{block (i,j)}$ is a submatrix of $K^{bU}$ consists of $(i-1)b+1$ to $ib$-th rows and $(j-1)b+1$ to $jb$-th columns. This is a block-wise extension of $K^U$ introduced in Eq. \ref{def_KU}. 
Fig. \ref{fig:toy_online}A contrasts $K^{bU}$ against full kernel $K$ and upper-triangular kernel $K^U$. 
Using $K^{bU}$, $f_b^{(n)}$ is written as follows:
\begin{proposition} \label{minibatch_prop}
    The predictor $f_b^{(n)} (X_*) = W_{n+1} \phi(X_*)$ obtained from $n$-steps of mini-batch learning with Eq. \ref{def_minibatch_sgd_kr} using data $X = [X_1, ..., X_n]$ and $Y = [Y_1, ..., Y_n]$ follows
    \begin{align} \label{eq_minibatch_predictor}
        f_b^{(n)} (X_*) = Y \left( \tfrac{1}{\eta} I + K^{bU} (X,X) \right)^{-1} K(X, X_*),
    \end{align}
    where $K^{bU}$ is block-wise upper triangular matrix defined in Eq. \ref{def_KbU}.
\end{proposition}
\begin{proof}
From a parallel arguments with Lemma \ref{lemma_online_kr_closed} (under $\gamma = 0$), we have
\begin{align}
    \Phi_i^T \prod_{j=i+1}^n \left( I - \eta \Phi_j \Phi_j^T \right) \Phi_* 
    = \left[ \left( I + \eta K^{bU} \right)^{-1} K(X, X_*) \right]_{block(i)}
\end{align}
Thus, we have
\begin{align}
    f_b^{(n)} (X_*) 
    &= \sum_{i=1}^n \eta Y_i \Phi_i^T \prod_{j=i+1}^n \left( I - \eta \Phi_j \Phi_j^T \right) \Phi_*
    \nonumber \\
    &= Y \left( \tfrac{1}{\eta} I + K^{bU} (X,X) \right)^{-1} K(X, X_*).
\end{align}
\end{proof}
Notably, unlike the sample-by-sample online learning, the upper triangular components in the block diagonal components are also missing under the mini-batch setting. It makes intuitive sense because samples learned together in a mini-batch cannot interact with each other.

\begin{figure}[t!]
  \begin{center}
    \centerline{\includegraphics[width=1.0\columnwidth]{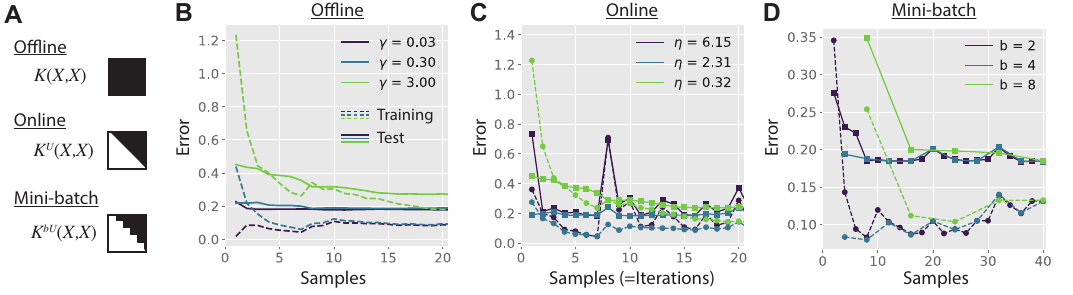}}
    \caption{ Illustration of learning curves under a random projection kernel $k(x, x') = \phi_J(x)^T \phi_J(x')$ with $\phi_J(x) = \tanh(J\vx)$. 
    \textbf{(A)} Visualization of kernels $K$, $K^U$, and $K^{bU}$.
    \textbf{(B)} Offline learning curves under various regularization amplitude $\gamma$. 
    \textbf{(C)} Online kernel regression learning curves. Points are error under iterative SGD updates $\vw \leftarrow \vw - \eta (\vw \phi_J(x) - y) \phi_J (x)^T$, while lines are offline-like expressions from Eq. \ref{eq_okr_closed_zero_reg}. 
    Learning rate $\eta$ were chosen as $\eta = 1/(\gamma + k_o)$ where $k_o \equiv E[k(x,x)]$ for $\gamma = \{ 0.03, 0.3, 3.0 \}$ for comparison with panel B. 
    \textbf{(D)} Mini-batch learning curves under mini-batch size $b=\{2,4,8\}$. Points are mini-batch SGD update while lines are offline-like expressions obtained from Eq. \ref{eq_minibatch_predictor}. 
    In B-D, each trajectory represents a single learning curve (without any averaging). We used the same random GP task illustrated in Fig. \ref{fig:kernel_reg}, and we generated $\phi_J(x)$ by randomly sampling elements of a 100-dimensional vector $J$ from a Gaussian distribution $\gauss(0,1)$. 
    }
    \label{fig:toy_online}
  \end{center}
\end{figure}

\subsection{Dynamics of target shift in online kernel regression} \label{subsec_target_shift_dynamics_kr} 
From the definition (Eq. \ref{eq_def_target_shift_kr}), the target shift after $(n+1)$-th update follows
\begin{align}
    Y^e_{n+1}
    &= Y_{n+1} \left( \tfrac{1}{\eta} I + K^U_{n+1} \right)^{-1} \left( \gamma I + K_{n+1} \right)
    \nonumber \\
    &= \begin{pmatrix} Y_n & \vy_{n+1} \end{pmatrix}
    \begin{pmatrix}
        \tfrac{1}{\eta} I + K^U_n & k(X_n, \vx_{n+1}) \\
        0 & 1/\eta \\
    \end{pmatrix}^{-1}
    \begin{pmatrix}
        \gamma I + K_n & k(X_n, \vx_{n+1}) \\
        k(X_n, \vx_{n+1})^T & \gamma + k_{n+1} \\
    \end{pmatrix}.
\end{align}
For brevity, here we defined
$K^U_n \equiv K^U (X_n, X_n)$ and $k_{n+1} \equiv k(\vx_{n+1}, \vx_{n+1})$.
Using the block matrix inverse property, we have
\begin{align}
    \begin{pmatrix}
        \tfrac{1}{\eta} I + K^U_n & k(X_n, \vx_{n+1}) \\
        0 & 1/\eta
    \end{pmatrix}^{-1}
    = \begin{pmatrix}
        \left( \tfrac{1}{\eta} I + K^U_n \right)^{-1} & - \eta \left( \tfrac{1}{\eta} I + K^U_n \right)^{-1} k(X_n, \vx_{n+1}) \\
        0 & \eta \\
    \end{pmatrix}.
\end{align}
Thus, $Y^e_{n+1}$ is written as
\begin{align}
    Y^e_{n+1}
    &= \begin{pmatrix}
        Y_n \left( \tfrac{1}{\eta} I + K^U_n \right)^{-1}, & - \eta \left[ f^{(n)}_{on} (\vx_{n+1}) - \vy_{n+1} \right]
    \end{pmatrix}
    \begin{pmatrix}
        \gamma I + K_n & k(X_n, \vx_{n+1}) \\
        k(X_n, \vx_{n+1})^T & \gamma + k_{n+1} \\
    \end{pmatrix}
    \nonumber \\
    & = \begin{pmatrix}
        Y^e_n - \eta \left[ f^{(n)}_{on} (\vx_{n+1}) - \vy_{n+1} \right] k(\vx_{n+1}, X_n), & 
        \vy_{n+1} - (\eta [\gamma + k_{n+1}] - 1) \left[f^{(n)}_{on} (\vx_{n+1}) - \vy_{n+1}\right]
    \end{pmatrix},
\end{align}
where
$f^{(n)}_{on} (\vx_{n+1}) \equiv f_{on}(\vx_{n+1}; X_n, Y_n) = Y_n (\tfrac{1}{\eta} I + K^U_n)^{-1} k(X_n, \vx_{n+1})$ from Eq. \ref{eq_okr_closed_zero_reg}. By denoting the prediction error of the online predictor by 
\begin{align}
    \bm{e}_{n+1} \equiv f^{(n)}_{on} (\vx_{n+1}) - \vy_{n+1},
\end{align}
we obtain Eq. \ref{eq_effective_target_change} in the main text. 

The update for the new sample ($i = n+1$) reveals an interesting limit. As $\eta \to 0$, the effective target $\vy^e_{n+1}(n+1)$ converges to $f_{\text{on}}(\vx_{n+1}; X_n, Y_n)$. Intuitively, if the learning rate is negligible, the model does not "accept" the new label $\vy_{n+1}$; instead, its effective target is simply its own current prediction, resulting in no change to the model parameters. Conversely, the $(n+1)$-th sample experiences no target shift ($\vy^e_{n+1} = \vy_{n+1}$) when the learning rate is tuned to $\eta = 1/(\gamma+k_{n+1})$, effectively matching the online and offline diagonal regularization.

\subsection{Dynamics of target correction in online kernel regression} \label{subsec_target_correction_dynamics_kr} 
With a parallel calculation with Appendix \ref{subsec_target_shift_dynamics_kr}, we can derive the dynamics of target correction $Y^c_n$. 
Corrected target in the presence of $(n+1)$ data points $X_{n+1}, Y_{n+1}$ is 
\begin{align}
    Y^c_{n+1}
    &= Y_{n+1} \left( \gamma I + K_{n+1} \right)^{-1} \left( \tfrac{1}{\eta} I + K^U_{n+1} \right)
    \nonumber \\
    &= \begin{pmatrix} Y_n & \vy_{n+1} \end{pmatrix}
    \begin{pmatrix}
        \gamma I + K_n & k(X_n, \vx_{n+1}) \\
        k(\vx_{n+1}, X_n) & \gamma + k_{n+1} \\
    \end{pmatrix}^{-1}
    \begin{pmatrix}
        \tfrac{1}{\eta} I + K^U_n & k(X_n, \vx_{n+1}) \\
        0 & 1/\eta\\
    \end{pmatrix},
\end{align}
where $K_n = K(X_n, X_n)$ and $k_{n+1} = k(\vx_{n+1}, \vx_{n+1})$ as before. 
Let us introduce scalar variables $q$ and $\rho$ by
\begin{align}
    q \equiv k(\vx_{n+1}, X_n) \left( \gamma I + K_n \right)^{-1} k(X_n, \vx_{n+1}), \quad
    \rho \equiv \frac{1}{\gamma + k_{n+1} - q}.
\end{align}
Then, the inverse matrix is rewritten as
\begin{align}
    &\begin{pmatrix}
        \gamma I + K_n & k(X_n, \vx_{n+1}) \\
        k(\vx_{n+1}, X_n) & \gamma + k_{n+1} \\
    \end{pmatrix}^{-1}
    \nonumber \\
    &= \begin{pmatrix}
        ( \gamma I + K_n )^{-1} \left( I + \rho k(X_n, \vx_{n+1}) k(\vx_{n+1}, X_n) ( \gamma I + K_n )^{-1} \right)
        & -\rho ( \gamma I + K_n )^{-1} k(X_n, \vx_{n+1}) \\
        -\rho  k(\vx_{n+1}, X_n) ( \gamma I + K_n )^{-1} 
        & \rho\\
    \end{pmatrix}.
\end{align}
Therefore, the corrected target follows
\begin{align}
    Y^c_{n+1}
    &= \begin{pmatrix} 
    Y_n (\gamma I + K_n)^{-1} + \rho \bm{e}^{\off}_{n+1} k(\vx_{n+1}, X_n) (\gamma I + K_n)^{-1}, &
    -\rho \bm{e}^{\off}_{n+1}
    \end{pmatrix}
    \begin{pmatrix}
        \tfrac{1}{\eta} I + K^U_n & k(X_n, \vx_{n+1}) \\
        0 & 1/\eta\\
    \end{pmatrix}
    \nonumber \\
    &= \begin{pmatrix}
        Y^c_n + \rho \bm{e}^{\off}_{n+1} k(\vx_{n+1}, X_n) (\gamma I + K_n)^{-1} (\tfrac{1}{\eta} I + K^U_n), &
        \vy_{n+1} + \left[ 1 + \rho(q - 1/\eta) \right] \bm{e}^{\off}_{n+1}
    \end{pmatrix},
\end{align}
where $\bm{e}^{\off}_{n+1}$ is the prediction error of the offline estimator defined by 
\begin{align}
\bm{e}^{\off}_{n+1} \equiv f_{\off} (\vx_{n+1}; X_n, Y_n) - \vy_{n+1}.
\end{align}
Denoting 
\begin{align} \label{eq_def_c1_c2}
    c_q \equiv \frac{\gamma + k_{n+1} - 1/\eta}{\gamma + k_{n+1} - q}, \quad
    C_K \equiv (\gamma I + K_n)^{-1} (\tfrac{1}{\eta} I + K^U_n),
\end{align}
we have
\begin{align}
    \vy^c_i (n+1) = \begin{cases}
        \vy^c_i (n) + \rho \bm{e}^{\off}_{n+1} \left[ k(\vx_{n+1}, X_n) C_K \right]_i & i \leq n \\
        \vy_{n+1} + c_q \bm{e}^{\off}_{n+1}, & i=n+1
    \end{cases}
\end{align}
Thus, the change in target correction depends on the prediction error of the offline predictor on the new sample $(\vx_{n+1}, \vy_{n+1})$, $\bm{e}^{\off}_{n+1}$. If the offline regressor $f^{(n)}_{\off}$ predicts the next data perfectly, there is no additional target correction, whereas in the presence of non-zero error, previous samples experience correction that depends on its similarity with the new input $k(\vx_{n+1}, \vx_i)$.

\subsection{ Discrepancy minimization formulation of target correction} \label{subsec_lossmin_tc}
We consider an optimization problem where we aim to optimize the target for online learning $Z$ such that the learned function from online learning becomes close to the function learned in the offline manner with the true target $Y$.
Because online and offline estimators are written as
\begin{subequations}
\begin{align}
    f_{on} (\vx; X, Z) &= Z \left( \tfrac{1}{\eta} I + K^U(X,X) \right)^{-1} k(X, \vx),
    \\
    f_{\off} (\vx; X, Y) &= Y \left( \gamma I + K(X,X) \right)^{-1} k(X, \vx),
\end{align}
\end{subequations}
minimization of the distance between the two functions is written as:
\begin{align}
    \ell(Z) &= \frac{1}{2} \llv Z (\tfrac{1}{\eta}I + K^U(X,X))^{-1} k(X, \cdot) - Y (\gamma I + K(X,X))^{-1} k(X, \cdot) \rrv_{\mathcal{H}}^2, 
\end{align}
where $\llv \cdot \rrv_{\mathcal{H}}$ denotes the RKHS norm. 
Assuming that $K(X,X)$ is invertible, the global optimization of $Z$ with respect to this loss yield
\begin{align}
    Z = Y \left( \gamma I + K (X,X) \right)^{-1} \left( \tfrac{1}{\eta} I + K^U (X, X) \right),
\end{align}
which recovers the target correction solution (Eq. \ref{eq_def_target_correction_kr}).

\subsection{Proof of Proposition \ref{prop_iter_correction}} \label{subsec_iter_c_proof}
Optimization of the loss (Eq. \ref{eq_iter_correction_loss}) with respect to $Z_{new}$ yields
\begin{align}
    Z_{new} = \left(G - Z_{past} B \right) C^{-1}, 
\end{align}
where matrices $B \in \Real^{(n-b) \times b}, C \in \Real^{b\times b}$, and $G \in \Real^{d_y \times b}$ are defined as
\begin{subequations}
\begin{align}
    \begin{pmatrix} * & B \\ B^T & C \end{pmatrix}
    &\equiv \left( \tfrac{1}{\eta} I + K^U_{tot} \right)^{-1} (\gamma_o I + K_{tot}) \left( \tfrac{1}{\eta} I + K^L_{tot} \right)^{-1} \\
    \begin{pmatrix} * & G \end{pmatrix}
    &\equiv Y (\gamma I + K_{tot})^{-1} K \left( \tfrac{1}{\eta} I + K^L_{tot} \right)^{-1},
\end{align}
\end{subequations}
where we introduced $K_{tot} = K(X_{tot}, X_{tot})$, $K^U_{tot} = K^U(X_{tot},X_{tot})$, and $K^L_{tot} = K^L(X_{tot},X_{tot})$ for brevity. $K^L = (K^U)^T$ is the strictly lower-triangular matrix generated from $K$, and $*$ terms in the matrices are terms we don't need to derive. 
Using the properties of block matrix inverse, from a straightforward linear algebra (see Appendix \ref{subsec_BCG_derivation}), $B, C$, and $G$ are written as
\begin{subequations}
\begin{align}
    B &= \left( \tfrac{1}{\eta} I + K^U_{pp} \right)^{-1} K_{pn} \left( I - (\tfrac{1}{\eta} I + K^U_{nn})^{-1} (\gamma_o I + K_{nn}) \right) \left( \tfrac{1}{\eta} I + K^L_{nn} \right)^{-1}, \\
    C &= \left( \tfrac{1}{\eta} I + K^U_{nn} \right)^{-1} (\gamma_o I + K_{nn}) \left( \tfrac{1}{\eta} I + K^L_{nn} \right)^{-1}, \\
    G &= \left( Y_{new} + \gamma \left[ Y_{past} (\gamma I + K_{pp}) ^{-1} K_{pn} - Y_{new} \right] Q^{-1} \right) (\tfrac{1}{\eta} I + K^L_{nn})^{-1}, 
\end{align}
\end{subequations}
where $K_{pp} = K(X_{past}, X_{past})$, $K_{pn} = K(X_{past}, X_{new})$, and $K_{nn} = K(X_{new}, X_{new})$. Matrix $Q$ in the third line is defined as
\begin{align}
    Q \equiv \left( \gamma I + K_{nn} \right) - K_{np} (\gamma I + K_{pp})^{-1} K_{pn}.
\end{align}
Let us further denote $X_{past}$ as $X_p$ and $X_{new}$ as $X_n$ for brevity, and use the same notation for $Y$ and $Z$.
Then, substituting matrices $B$, $C$, and $G$ in $Z_{new} = (G - Z_p B) C^{-1}$ with the expressions above, we have
\begin{align}
    Z_{new} &= (G - Z_p B) C^{-1}
    \nonumber \\
    &= \left( Y_n + \gamma [Y_p (\gamma I + K_{pp})^{-1} K_{pn} - Y_n] Q^{-1} \right) (\gamma_o I + K_{nn})^{-1} (\tfrac{1}{\eta} I + K^U_{nn}) 
    \nonumber \\
    &\quad - Z_p (\tfrac{1}{\eta} I + K^U_{pp})^{-1} K_{pn} \left( I - (\tfrac{1}{\eta} I + K^U_{nn})^{-1} (\gamma_o I + K_{nn}) \right) (\gamma_o I + K_{nn})^{-1} (\tfrac{1}{\eta} I + K^U_{nn})
    \nonumber \\
    &= \left( Y_ n + \gamma \left[ f_{\off} (X_n; X_p, Y_p) - Y_n \right] Q^{-1} \right) (\gamma_o I + K_{nn})^{-1} (\tfrac{1}{\eta} I + K^U_{nn}) 
    \nonumber \\
    &\quad - f_{on} (X_n; X_p, Z_p) \left[(\gamma_o I + K_{nn})^{-1} (\tfrac{1}{\eta} I + K^U_{nn}) - I \right]
    \nonumber \\
    &= Y_n + \gamma \left[ f_{\off} (X_n; X_p, Y_p) - Y_n \right] Q^{-1} (\gamma_o I + K_{nn})^{-1} (\tfrac{1}{\eta} I + K^U_{nn}) 
    \nonumber \\
    &\quad + \left( Y_n - f_{on} (X_n; X_p, Z_p) \right) \left[(\gamma_o I + K_{nn})^{-1} (\tfrac{1}{\eta} I + K^U_{nn}) - I \right]. 
\end{align}
In the third line, we used Eqs. \ref{eq_batch_kernel_regression} and \ref{eq_okr_closed_zero_reg} to get
\begin{align}
    Y_p (\gamma I + K_{pp})^{-1} K_{pn} = f_{\off} (X_n; X_p, Y_p), \quad
    Z_p (\tfrac{1}{\eta} I + K^U_{pp})^{-1} K_{pn} = f_{on} (X_n; X_p, Z_p). 
\end{align}
Let us define $C_{\off}$ and $C_{on}$ by 
\begin{subequations}
\begin{align}
  C_{on} &\equiv \left((\gamma_o I + K_{nn})^{-1} (\tfrac{1}{\eta} I + K^U_{nn}) - I \right),
  \\
  C_{\off} &\equiv \gamma \left( [\gamma I + K_{nn}] - K_{np} [\gamma I + K_{pp}]^{-1} K_{pn} \right)^{-1} (\gamma_o I + K_{nn})^{-1} (\tfrac{1}{\eta} I + K^U_{nn}).
\end{align}
\end{subequations}
Then, we obtain the simple expression from the main text:
\begin{align}
    Z_{new} = Y_n + (Y_n -f_{on}(X_n; X_p,Z_p))C_{on} + (f_{\off} (X_n; X_p, Y_p) - Y_n) C_{\off}.
\end{align}


\subsection{Derivation of matrices $B$, $C$, $G$} \label{subsec_BCG_derivation}
Here, we detail the derivation of matrices $B$, $C$, and $G$ for completeness. 
From the block matrix inverse formula, the inverse of the upper triangular matrix $\tfrac{1}{\eta} I + K^U$ follows
\begin{align} \label{eq_KU_inv_block_epr}
    \left( \tfrac{1}{\eta} I + K^U \right)^{-1}
    &= \begin{pmatrix}
    \tfrac{1}{\eta} I + K^U_{pp} & K_{pn} \\
    0 & \tfrac{1}{\eta} I + K^U_{nn}
    \end{pmatrix}^{-1}
    \nonumber \\
    &= \begin{pmatrix}
     ( \tfrac{1}{\eta} I + K^U_{pp} )^{-1} &
     - ( \tfrac{1}{\eta} I + K^U_{pp} )^{-1} K_{pn} ( \tfrac{1}{\eta} I + K^U_{nn} )^{-1} \\
     0 & 
     ( \tfrac{1}{\eta} I + K^U_{nn} )^{-1}
    \end{pmatrix},
\end{align}
where $K_{pp} \in \Real^{(n-b) \times (n-b)}$ is the kernel matrix of previous $n-b$ samples, $K_{nn} \in \Real^{b \times b}$ is the kernel matrix of the new $b$ samples, and $K_{pn} \in \Real^{(n-b) \times b}$ is the cross term. 
Similarly, the inverse of the lower triangular matrix $\tfrac{1}{\eta} I + K^L$ is decomposed as
\begin{align}
  \left( \tfrac{1}{\eta} I + K^L \right)^{-1}
  &= \begin{pmatrix}
  \tfrac{1}{\eta} I + K^L_{pp} & 0 \\
  K_{np} & \tfrac{1}{\eta} I + K^L_{nn}
  \end{pmatrix}^{-1}
  \nonumber \\
  &= \begin{pmatrix}
  ( \tfrac{1}{\eta} I + K^L_{pp} )^{-1} &
  0 \\
  - ( \tfrac{1}{\eta} I + K^L_{nn} )^{-1} K_{np} (  \tfrac{1}{\eta} I + K^L_{pp} )^{-1} & 
  ( \tfrac{1}{\eta} I + K^L_{nn} )^{-1}
  \end{pmatrix}.
\end{align}
Therefore, matrices $B$ and $C$ follow
\begin{align}
 \begin{pmatrix} * & B \\ B^T & C
 \end{pmatrix}
 &= \left( \tfrac{1}{\eta} I + K^U \right)^{-1} \left( \gamma_o I+ K \right) \left( \tfrac{1}{\eta} I + K^L \right)^{-1}
 \nonumber \\
 &= \left( \tfrac{1}{\eta} I + K^U \right)^{-1} 
 \begin{pmatrix} 
 \,\, * & K_{pn} \left( \tfrac{1}{\eta} I + K^L_{nn} \right)^{-1} \\
 \,\, * & \left( \gamma_o I + K_{nn} \right) \left( \tfrac{1}{\eta} I + K^L_{nn} \right)^{-1}
 \end{pmatrix}
 \nonumber \\
 &= \begin{pmatrix} 
 \,\, * & 
 \left( \tfrac{1}{\eta} I + K^U_{pp} \right)^{-1} K_{pn} \left( I - ( \tfrac{1}{\eta} I + K^U_{nn} )^{-1} \left( \gamma_o I + K_{nn} \right) \right) ( \tfrac{1}{\eta} I + K^L_{nn} )^{-1} 
 \\
 \,\, * & 
 ( \tfrac{1}{\eta} I + K^U_{nn} )^{-1}  
\left( \gamma_o I + K_{nn} \right)
 ( \tfrac{1}{\eta} I + K^L_{nn} )^{-1}  
 \end{pmatrix},
\end{align}
where $*$ represents terms we omitted, as we do not need to estimate for the iterative correction estimation. 
Similarly, on $G$ term, denoting
\begin{align}
    Q \equiv (\gamma I + K_{nn}) - K_{np} (\gamma I + K_{pp})^{-1} K_{pn}, 
\end{align}
$(\gamma I + K)^{-1}$ is rewritten as
\begin{align} \label{eq_K_inv_block_epr}
  (\gamma I + K)^{-1}
  &= \begin{pmatrix}
      \gamma I + K_{pp} & K_{pn} \\
      K_{np} & \gamma I + K_{nn} 
  \end{pmatrix}^{-1}
  \nonumber \\
  &= \begin{pmatrix}
  (\gamma I + K_{pp})^{-1} \left( I + K_{pn} Q^{-1} K_{np} (\gamma I + K_{pp})^{-1} \right) &
  - (\gamma I + K_{pp})^{-1} K_{pn} Q^{-1} \\
  - Q^{-1} K_{np} \left( \gamma I + K_{pp} \right)^{-1} &
  Q^{-1}
  \end{pmatrix},
\end{align}
and $K (\tfrac{1}{\eta} I + K^L )^{-1}$ follows
\begin{align}
    K \left( \tfrac{1}{\eta} I + K^L \right)^{-1}
    &= \begin{pmatrix}
    K_{pp} & K_{pn} \\ K_{np} & K_{nn} 
    \end{pmatrix}
    \begin{pmatrix}
     (\tfrac{1}{\eta} I + K^L_{pp})^{-1} & 0 \\
     - (\tfrac{1}{\eta} I + K^L_{nn})^{-1} K_{np} (\tfrac{1}{\eta} I + K^L_{pp})^{-1} &
     (\tfrac{1}{\eta} I + K^L_{nn})^{-1} 
    \end{pmatrix}
    \nonumber \\
    &= \begin{pmatrix}
      \quad * & K_{pn} ( \tfrac{1}{\eta} I + K^L_{nn} )^{-1} \\
      \quad * & K_{nn} ( \tfrac{1}{\eta} I + K^L_{nn} )^{-1} 
    \end{pmatrix}.
\end{align}
Thus, $G$ in $(*, G) = Y (\gamma I + K)^{-1} K (\tfrac{1}{\eta}I + K^L)$ is rewritten as
\begin{align}
    G
    =& 
    \left[ Y_p (\gamma I + K_{pp})^{-1} \left( I + K_{pn} Q^{-1} K_{np} (\gamma I + K_{pp})^{-1} \right) - Y_n Q^{-1} K_{np} (\gamma I + K_{pp})^{-1} \right] K_{pn} (\tfrac{1}{\eta} I + K^L_{nn})^{-1} 
    \nonumber \\
    & + \left( - Y_p (\gamma I + K_{pp})^{-1}K_{pn} + Y_n \right) Q^{-1} K_{nn} (\tfrac{1}{\eta} I + K^L_{nn})^{-1} .
\end{align}
While the expression of matrix $G$ appears complicated, we can simplify the expression by first rearranging the equations by separating $Y_p$ and $Y_n$ dependent terms:
\begin{align}
    G &= 
    Y_p (\gamma I + K_{pp})^{-1} \left( K_{pn} - K_{pn}Q^{-1} \left[ K_{nn} - K_{np} (\gamma I + K_{pp})^{-1} K_{pn} \right] \right) (\tfrac{1}{\eta} I + K^L_{nn})^{-1}
    \nonumber \\
    &\quad\;+ Y_n Q^{-1} \left[ K_{nn} - K_{np} (\gamma I + K_{pp})^{-1} K_{pn} \right] (\tfrac{1}{\eta} I + K^L_{nn})^{-1}
    \nonumber \\
    & = Y_p (\gamma I + K_{pp})^{-1} K_{pn} \gamma Q^{-1} (\tfrac{1}{\eta} I + K^L_{nn})^{-1}
    + Y_n (I - \gamma Q^{-1})
     (\tfrac{1}{\eta} I + K^L_{nn})^{-1} 
    \nonumber \\
    &= \left( Y_n + \gamma \left[ Y_p (\gamma I + K_{pp})^{-1} K_{pn} - Y_n \right] Q^{-1} \right) (\tfrac{1}{\eta} I + K^L_{nn})^{-1}.
\end{align}
In the third line, we used 
$K_{nn} - K_{np} (\gamma I + K_{pp})^{-1} K_{pn} = Q - \gamma I$. 

\subsection{Online-offline gap under iterative target correction}
Under the iterative target correction, while stability and causality are restored, equivalence between online and offline predictors (Theorem \ref{target_correction_theorem}) no longer holds in general. 
This motivates us to analyze the progression of the difference between online and offline predictors $f_{on}$ and $f_{\off}$ under the iterative algorithm. 
Below, we first note that the gap has a simple expression. We then show that in the non-i.i.d. online continual learning setting, the gap is bounded by the cumulative error under the iterative offline learning. 

As in the main text, we denote a set of previous samples as $X_p, Y_p$, a set of $b$ samples newly introduced for iterative correction as $(X_n, Y_n)$, and the total samples as $(X_{tot}, Y_{tot})$ (i.e., $X_{tot} = [X_p, X_n]$ and $Y_{tot} = [Y_p, Y_n]$). 
The online/offline gap after a one-step iterative correction is then written as follows. 

\begin{lemma}[Online/offline gap expression]
    Consider a one-step iterative target correction with Prop. \ref{prop_iter_correction} from $(X_p, Y_p, Z_p)$ using new samples $(X_n, Y_n)$. The online/offline gap under one-step update follows:
    \begin{align}
    &f_{on} (\,\cdot\,; X_{tot}, Z_{tot}) - f_{\text{off}} (\,\cdot\,; X_{tot}, Y_{tot})
    \nonumber \\
    &= \left[ g_p(\,\cdot\,) - g_p (X_n) (\gamma_o I + K_{nn})^{-1} k(X_n, \,\cdot\,) \right]
    - E_n Q^{-1} \left[ r_n(\cdot) - r_n(X_n)  (\gamma_o I + K_{nn})^{-1} k(X_n, \,\cdot\,)\right],
    \end{align}
    where $K_{nn} \equiv k(X_n, X_n)$, Matrices $E_n$, $Q$ are defined as
    \begin{subequations}
    \begin{align}
        E_n &\equiv Y_n - f_{\off} (X_n; X_p, Y_p) \\
        Q &\equiv (\gamma I + K_{nn}) - K_{np}(\gamma I + K_{pp})^{-1} K_{pn},
    \end{align}
    \end{subequations}
    and functions $r_n (\,\cdot\,)$, $g_p (\,\cdot\,)$ in the RKHS are defined as 
    \begin{subequations}
    \begin{align}
        r_n (\,\cdot\,) &\equiv k(X_n, \,\cdot\,) - K_{np} (\gamma I + K_{pp})^{-1} k(X_p, \,\cdot\,) \\
        g_p (\,\cdot\,) &\equiv f_{on} (\,\cdot\,; X_p, Z_p) - f_{\off} (\,\cdot\,; X_p, Y_p). 
    \end{align}
    \end{subequations}
\end{lemma}

\begin{proof}
From Eq. \ref{eq_KU_inv_block_epr}, $f_{on}(\,\cdot\,; X_{tot}, Z_{tot})$ is rewritten as
\begin{align} \label{eq_fon_tot_tmp1}
    f_{on}(\,\cdot \,; X_{tot}, Z_{tot})
    &= f_{on} (\,\cdot \,; X_p, Z_p) + \left( Z_n - Z_p (\tfrac{1}{\eta} I + K^U_{pp})^{-1} K_{pn} \right) (\tfrac{1}{\eta} I + K^U_{nn})^{-1} k(X_n, \,\cdot\,)
    \nonumber \\
    &= f_{on} (\,\cdot \,; X_p, Z_p)
    + f_{on} \left( \,\cdot\, ; Z_n - f_{on}(X_n; X_p, Z_p) \right).
\end{align}
Similarly, from Eq. \ref{eq_K_inv_block_epr}, $f_{\off}(\,\cdot\,; X_{tot}, Y_{tot})$ is rewritten as
\begin{align}
  &f_{\off}(\,\cdot\,; X_{tot}, Y_{tot})
  \nonumber \\
  &= f_{\off} (\,\cdot\,; X_p, Y_p) + \left( Y_n - Y_p (\gamma I + K_{pp})^{-1} K_{pn} \right) Q^{-1} (k(X_n, \,\cdot\,) - K_{np} (\gamma I + K_{pp})^{-1} k(X_p, \,\cdot\,) )
  \nonumber \\
  &= f_{\off} (\,\cdot\,; X_p, Y_p) + E_n Q^{-1} r_n (\,\cdot\,),
\end{align}
where $Q$, $E_n$, and $r_n (\,\cdot\,)$ are defined above.
Using the definition of $Z_n$ (in Eqs. \ref{eq_citer_Znew} - \ref{eq_def_Con_Coff}),  
\begin{align}
    &Z_n - f_{on}(X_n; X_p, Z_p)
    \nonumber \\
    &= \left( \left[ Y_n - f_{on}(X_n;X_p,Z_p) \right] + [f_{\off}(X_n;X_p,Y_p) - Y_n] \gamma Q^{-1} \right) (\gamma_o I + K_{nn})^{-1} (\tfrac{1}{\eta} I + K^U_{nn})
    \nonumber \\
    &= \left(  E_n(I- \gamma Q^{-1}) - g_p (X_n) \right) (\gamma_o I + K_{nn})^{-1} (\tfrac{1}{\eta} I + K^U_{nn}),
\end{align}
where, in the second line, we used 
$g_p(X_n) = f_{on} (X_n; X_p, Z_p) - f_{\off} (X_n; X_p, Y_p)$.
By replacing $Z_n - f_{on}(X_n; X_p, Z_p)$ in Eq. \ref{eq_fon_tot_tmp1} with the equation above and use $I-\gamma Q^{-1} = Q^{-1} (Q - \gamma I) = Q^{-1} r_n(X_n)$, we have
\begin{align}
    f_{on}(\,\cdot \,; X_{tot}, Z_{tot})
    = f_{on} (\,\cdot \,; X_p, Z_p)
    + \left(  E_n Q^{-1} r_n(X_n) - g_p (X_n) \right) (\gamma_o I + K_{nn})^{-1} k(X_n, \,\cdot\,). 
\end{align}
Therefore, $f_{on} (\,\cdot\,; X_{tot}, Z_{tot}) - f_{\off} (\,\cdot\,; X_{tot}, Y_{tot})$ is rewritten as
\begin{align}
    &f_{on} (\,\cdot\,; X_{tot}, Z_{tot}) - f_{\off} (\,\cdot\,; X_{tot}, Y_{tot})
    \nonumber \\
    &= \left[ g_p(\,\cdot\,) - g_p (X_n) (\gamma_o I + K_{nn})^{-1} k(X_n, \,\cdot\,) \right]
    - E_n Q^{-1} \left[ r_n(\cdot) - r_n(X_n)  (\gamma_o I + K_{nn})^{-1} k(X_n, \,\cdot\,)\right]
\end{align}
\end{proof}

\subsection{Online-offline gap under iterative target correction in the general non-i.i.d setting (Proposition \ref{prop_iter_tc_bound_noniid})} \label{subsec_noniid_bound}
We first consider the upper bound of the online/offline gap under a general non-i.i.d. setting, which is relevant to online continual learning scenarios with distributional shift over samples. 
In this setting, the online/offline gap generally doesn't converge to zero under iterative target correction. It underlies the challenge of online continual learning where previously learned samples may not help or even impede subsequent learning. 
Nevertheless, we show that the online/offline gap under the iterative target correction is bounded by the error of the iteratively-updated offline predictor. This is in contrast with the gap under naive online learning with the true target, where the gap may explode depending on the sample order. 

\paragraph{Proposition \ref{prop_iter_tc_bound_noniid} (restated):} 
Consider iterative target correction with data $X_{1:n} = [X_1, ...,X_n]$ and $Y_{1:n} = [Y_1, ..., Y_n]$. 
Then, under $\gamma > 0$ and $\gamma_o > 0$, the online/offline gap $\epsilon^g$ is bounded by:
\begin{align}
    \epsilon^{g}_n \equiv
    &\llv f_{on}(\,\cdot\,; X_{1:n}, Z_{1:n}) - f_{\off} (\,\cdot\,; X_{1:n}, Y_{1:n}) \rrv_{\mathcal{H}}
    \leq \frac{1}{\sqrt{\gamma}} \sum_{i=1}^n \llv Y_i - f_{\off} (X_i; X_{1:i-1}, Y_{1:i-1})\rrv_{F}.
\end{align}

\begin{proof}
We first consider one-step bound at $i=n$. Let us denote $X_p = [X_1, ..., X_{n-1}]$, $Y_p = [Y_1, ..., Y_{n-1}]$, and $Z_p = [Z_1, ..., Z_{n-1}]$. 
From the lemma above, the gap is decomposed into two terms. 
From triangular inequality, 
\begin{align}
    \epsilon^{g}_n 
    &\leq
    \llv \left[ g_p(\,\cdot\,) - g_p (X_n) (\gamma_o I + K_{nn})^{-1} k(X_n, \,\cdot\,) \right] \rrv_{\mathcal{H}}
    \nonumber \\
    &\quad + \llv E_n Q^{-1} \left[ r_n(\cdot) - r_n(X_n)  (\gamma_o I + K_{nn})^{-1} k(X_n, \,\cdot\,) \right] \rrv_{\mathcal{H}}. 
\end{align}
The first term that depends on the previous discrepancy $g_p$ satisfies
\begin{align}
  \lVert g_p(\,\cdot\,) - g_p (X_n) (\gamma_o I + K_{nn})^{-1} k(X_n, \,\cdot\,) \rVert_{\mathcal{H}}
  \leq \lVert g_p \rVert_{\mathcal{H}}.
\end{align}
This is because denoting the sampling operator $\hat{S}_n : \mathcal{H} \to \Real^n$ by $\hat{S}_n g = g(X_n)^T$, we have
\begin{align} \label{eq_gp_gpnn_shrinkage}
g_p(\cdot) - g_p (X_n) (\gamma_o I + K_{nn})^{-1} k(X_n, \cdot) &= \left( I - \hat{S}_n^T (\gamma_o I + \hat{S}_n \hat{S}_n^T)^{-1} \hat{S}_n \right) g_p (\cdot)
\nonumber \\
&= \gamma_o (\gamma_o I + \hat{S}_n^T \hat{S}_n)^{-1} g_p(\cdot), 
\end{align}
and $\gamma_o (\gamma_o I + \hat{S}_n^T \hat{S}_n)^{-1} $ is a shrinking operator. 
By definition, $\llv g_p (\,\cdot\,) \rrv_{\mathcal{H}}$ follows
\begin{align}
    \llv g_p (\,\cdot\,) \rrv_{\mathcal{H}}
    = \llv f_{on} \left( \, \cdot \, ; X_{1:n-1}, Z_{1:n-1} \right) - f_{\off} \left( \, \cdot \, ; X_{1:n-1}, Y_{1:n-1} \right) \rrv_{\mathcal{H}}
    = \epsilon^g_{n-1}.
\end{align}
This means that, the discrepancy originated from previous samples do not increase under the iterative target correction.

However, the second term that scales with the current offline predictor error $E_n$ potentially increases the gap. From a parallel argument with Eq. \ref{eq_gp_gpnn_shrinkage}, this term is rewritten as
\begin{align}
    &\llv E_n Q^{-1} \left[ r_n(\,\cdot\,) - r_n(X_n)  (\gamma_o I + K_{nn})^{-1} k(X_n, \,\cdot\,)\right] \rrv^2_{\mathcal{H}}
    \leq \llv E_n Q^{-1} r_n(\,\cdot\,) \rrv^2_{\mathcal{H}}
\end{align}
From the definition of $r_n(\,\cdot\,)$, we have
\begin{align} \label{eq_En_Qinv_rn_bound}
    &\llv E_n Q^{-1} r_n(\,\cdot\,) \rrv^2_{\mathcal{H}}
    \nonumber \\
    &\leq \tr \left[ E_n Q^{-1} \left( K_{nn} - K_{np} (\gamma I + K_{pp})^{-1} K_{pn} - \gamma K_{np} (\gamma I + K_{pp})^{-2}  K_{pn} \right) Q^{-1} E_n^T \right] 
    \nonumber \\
    &\leq \tr \left[ E_n Q^{-1} \left( Q - \gamma K_{np} (\gamma I + K_{pp})^{-2}  K_{pn} - \gamma I \right) Q^{-1} E_n^T \right] 
    \nonumber \\
    & \leq \tr \left[ E_n Q^{-1} E_n^T \right]
    \leq \llv Q^{-1} \rrv_2 \llv E_n \rrv_F^2
\end{align}
In the third line, we used $Q = \gamma I + \left[ K_{nn} - K_{np}(\gamma I + K_{pp})^{-1} K_{pn} \right]$, and in the last line, we used $ Q - \gamma K_{np} (\gamma I + K_{pp})^{-2}  K_{pn} - \gamma I \preceq Q$. 
The norm $\llv Q^{-1} \rrv_2$ is upper-bounded by
\begin{align}
    \llv Q^{-1} \rrv_2
    = \llv \left( \gamma I + \left[ K_{nn} - K_{np}(\gamma I + K_{pp})^{-1} K_{pn} \right] \right)^{-1} \rrv_2 
    \leq \frac{1}{\gamma}. 
\end{align}
The last inequality holds because $ K_{nn} - K_{np}(\gamma I + K_{pp})^{-1} K_{pn}$ is positive semi-definite. Therefore, we have
\begin{align}
    \llv E_n Q^{-1} \left[ r_n(\,\cdot\,) - r_n(X_n)  (\gamma_o I + K_{nn})^{-1} k(X_n, \,\cdot\,)\right] \rrv_{\mathcal{H}}
    \leq \frac{1}{\sqrt{\gamma}} \llv E_n \rrv_F. 
\end{align}
Combining the upper bounds on the two terms and apply them recursively, we get
\begin{align}
     \llv f_{on}(\,\cdot\,; X_{1:n}, Z_{1:n}) - f_{\off} (\,\cdot\,; X_{1:n}, Y_{1:n}) \rrv_{\mathcal{H}}
    &\leq \frac{1}{\sqrt{\gamma}} \sum_{i=1}^n \llv E_i \rrv_F
    \nonumber \\
    &\leq \frac{1}{\sqrt{\gamma}} \sum_{i=1}^n \llv Y_i - f_{\off} \left( X_i; X_{1:i-1}, Y_{1:i-1} \right) \rrv_F.
\end{align}
\end{proof}
Thus, the online-offline discrepancy under iterative target correction is upper-bounded by the error of iterative offline learning, $\llv Y_i - f_{\off} \left( X_i; X_{1:i-1}, Y_{1:i-1} \right) \rrv_F$.
In general, this bound does not converge to zero in the $n \to \infty$ limit under arbitrary sample orderings. This is an inevitable feature of the non-i.i.d.\ setting, in which new samples may be adversarial to old ones, as exemplified by domain- and class-incremental continual learning.

By contrast, under vanilla online learning, from Eq.~\ref{eq_def_target_shift_kr}, the online-offline gap satisfies
\begin{align}
    \llv f_{\text{on}} (\,\cdot\,; X,Y) - f_{\off} (\,\cdot\,; X,Y) \rrv_{\mathcal{H}}
    = \llv (Y^e - Y ) (\gamma I + K(X,X))^{-1} k(X, \cdot)\rrv_{\mathcal{H}}
    \leq \frac{1}{2\sqrt{\gamma}} \llv Y^e - Y\rrv_F .
\end{align}
Furthermore, we have
\begin{align} \label{eq_naive_online_gap_bound}
    \llv Y^e - Y\rrv_F
    &\leq \llv Y (\tfrac{1}{\eta} I + K^U(X,X))^{-1} (\gamma I + K(X,X)) \rrv_F + \llv Y \rrv_F
    \nonumber \\
    &\leq \llv Y \rrv_F \left( \frac{\sigma_{\max} (K(X,X)) + \gamma}{ \sigma_{\min} (\tfrac{1}{\eta} I + K^U(X,X))} + 1\right),
\end{align}
where $\sigma_{\max}(M)$ and $\sigma_{\min}(M)$ denote the largest and smallest singular values of a matrix $M$, respectively.
Notably, because $\tfrac{1}{\eta} I + K^U_{nn}$ is non-normal, its smallest singular value is not bounded below by $1/\eta$. Instead, it can decay exponentially with the number of samples, which underlies the instability of vanilla online learning.
In contrast, under iterative target correction, the upper bound no longer scales with $\sigma_{\min}^{-1} (\tfrac{1}{\eta} I + K^U(X,X))$, supporting robust learning dynamics.

\subsection{Online-offline gap under iterative target correction in an i.i.d setting} \label{subsec_iid_bound}
We next consider the i.i.d. case to further examine the properties of iterative target correction. Although our framework primarily concerns online continual learning, where samples are drawn in a non-i.i.d. manner, convergence in the i.i.d. setting provides additional support for the validity of the iterative target correction algorithm.

\paragraph{Assumptions \ref{subsec_iid_bound}} 
In addition to i.i.d. sampling assumption, in this subsection, but only in this subsection, we introduce the following set of assumptions.
We set $d_y = 1$ and consider a finite dimensional feature $\phi : \mathcal{X} \to \Real^{d_h}$ with $d_h < \infty$ for kernel  $k(\vx,\vx') = \phi(\vx)^T \phi(\vx')$. We further assume that the features are normalized (i.e., $\llv \phi(\vx) \rrv_2 = 1$ for all $\vx \in \mathcal{X}$) and for the probability distribution $p(x)$ from which we pick samples, $E_{\vx \sim p(\vx)} [\phi(\vx) \phi(\vx)^T] = \Sigma \succeq \rho I_h$ with $\rho > 0$. 
Finally, target $y_i$ at trial $i$ are generated by $y_i = \bm{w}_*^T \phi(\vx)$. 

Before deriving the upper bound of the offline-online gap, we note that the error of the offline kernel regressions, $E_t = Y_t - f_{\off}(X_t; X_{1:t-1}, Y_{1:t-1})$, converges at $\Order(1/t)$ rate in the i.i.d. setting. Note that, under $d_y = 1$, $E_t$ is a $b$-dimensional vector.  

\begin{lemma} \label{lemma_Ebb_Et_bound}
    Under the Assumptions \ref{subsec_iid_bound} stated above, the following upper bound holds for the error under the offline kernel regression: 
    \begin{align} \label{eq_Ebb_Et_bound}
        \Ebb_{X_{1:t+1}} \left[ || E_{t+1} ||^2 \right]
        &\leq b \llv \vw_* \rrv^2 \frac{(1+\gamma)^2}{ \gamma (bt + 1)}.
    \end{align}
\end{lemma}
\begin{proof}
The proof follows from the leave-one-out bounds in \citet{zhang2003leave} which we describe for completeness. 
Let us denote 
\begin{align}
B \equiv \gamma I + \sum_{j=1}^{bt+1} \phi(\vx_j) \phi(\vx_j)^T, \quad 
A_{-i} \equiv \gamma I + \sum_{j=1, j\neq i}^{bt+1} \phi(\vx_j) \phi(\vx_j)^T = B - \phi (\vx_i) \phi (\vx_i)^T,
\end{align}
and $\phi_i \equiv \phi (\vx_o)$ for brevity. 
If we exchange the $i$-th sample with the $bt+1$-th sample, the error is written as
\begin{align}
  \epsilon_i \equiv \left( \vw_*^T \phi(\vx_i) - f_{\off} (\vx_i; X_{-i}, Y_{-i})\right)^2
  = \left( \gamma \vw_*^T A_{-i}^{-1} \phi(\vx_i) \right)^2,
\end{align}
where $(X_{-i}, Y_{-i})$ are the sets of training samples except the $i$-th sample. Using $\epsilon_i$, from exchangeability of i.i.d. samples, the error is rewritten as
\begin{align} \label{eq_Et_bound2_exchange}
    \Ebb_{X_{1:t}, \vx_*} \left[ (y_* - f_{\off} (\vx_*; X_{1:t}, Y_{1:t}) )^2\right]
    = \frac{1}{bt+1} \Ebb_{X_{1:t}, \vx_{bt+1}} \left[ \sum_{i=1}^{bt+1} \epsilon_i \right]
\end{align}
From $A_{-i}^{-1} 
= (B - \phi(\vx_i) \phi(\vx_i)^T)^{-1} 
= B^{-1} + \frac{B^{-1} \phi_i \phi_i^T B^{-1} }{1 - \phi_i^T B^{-1} \phi_i}$, 
we have
\begin{align}
    \epsilon_i 
    = \left( \gamma \vw_*^T A_{-i}^{-1} \phi(\vx_i) \right)^2
    = \frac{\gamma^2 (\vw_*^T B^{-1} \phi_i)^2}{ \left( 1 - \phi_i^T B^{-1} \phi_i \right)^2}, 
\end{align}
and because $B \succeq \gamma I + \phi_i \phi_i^T$, 
\begin{align}
    \phi_i^T B^{-1} \phi_i 
    \leq \phi_i^T \left( \gamma I + \phi_i \phi_i^T \right)^{-1} \phi_i 
    \leq \frac{1}{1 + \gamma}. 
\end{align}
Therefore, 
\begin{align}
    \epsilon_i 
    = \frac{\gamma^2 (\vw_*^T B^{-1} \phi_i)^2}{(1 - \phi_i^T B^{-1} \phi_i)^2}
    \leq (1 + \gamma)^2 \left( \vw_*^T B^{-1} \phi_i \right)^2.
\end{align}
Lastly, taking summation over $i$
\begin{align}
    \sum_{i=1}^{bt+1} \epsilon_i 
    &\leq (1 + \gamma)^2 \sum_{i=1}^{bt+1} \left( \vw_*^T B^{-1} \phi_i  \right)^2 = (1 + \gamma)^2 \vw_*^T B^{-1} \left( \sum_{i=1}^{bt+1} \phi_i \phi_i^T \right) B^{-1} \vw_*
    \nonumber \\
    &= (1 + \gamma)^2 \left( \vw_*^T B^{-1} \vw_* - \gamma \llv B^{-1} \vw_* \rrv^2 \right)
    \leq (1 + \gamma)^2 \vw_*^T B^{-1} \vw_*
    \leq \frac{ (1+\gamma)^2}{\gamma} \llv \vw_*\rrv^2. 
\end{align}
Substituting $\epsilon_i$ in Eq. \ref{eq_Et_bound2_exchange} in the expression above, we obtain
\begin{align}
    \Ebb_{X_{1:t}, \vx_*} \left[ (y_* - f_{\off} (\vx_*; X_{1:t}, Y_{1:t}) )^2\right] \leq \frac{ (1+\gamma)^2 \llv \vw_* \rrv^2 }{ \gamma (bt+1)}.
\end{align}
Because $E_t = Y_t - f_{\off}(X_t; X_{1:t-1}, Y_{1:t-1})$ is defined batch-wise (with batch size $b$), by multiplying the bound above with $b$, we obtain Eq. \ref{eq_Ebb_Et_bound}. 
\end{proof}

We also introduce the following Lemma on the expected step-wise shrinkage of the gap: 
\begin{lemma} \label{lemma_Rb_shrinkage}
    Let us denote $R^{(b)}$ by 
    \begin{align}
        R^{(b)} \equiv \gamma_o \left(\gamma_o I_h + \sum_{i=1}^b \phi(\vx_i) \phi(\vx_i)^T \right)^{-1}, \label{eq_def_Rb}
    \end{align}
    where $X = [\vx_1, ..., \vx_b]$ are i.i.d. samples from $p(\vx)$ and $I_h$ is the size $d_h$ identity matrix. Then, denoting the largest eigenvalue of matrix $M$ by $\lambda_{\max}[M]$, we have
    \begin{align}
        \lambda_{\max} (\Ebb_{X} [(R^{(b)})^2]) \leq 1 - \frac{\rho}{1 + \gamma_o}
    \end{align}
\end{lemma}
\begin{proof}
Because $R^{(b)}$ satisfies $0 \preceq R^{(b)} \preceq I$, $(R^{(b)})^2 \preceq R^{(b)}$.
Additionally, because $\sum_{i=1}^b \phi(\vx_i) \phi(\vx_i)^T$ is positive semi-definite, we have
\begin{align}
R^{(b)} = \gamma_o \left( \gamma_o I_h + \sum_{i=1}^b \phi(\vx_i) \phi(\vx_i)^T \right)^{-1} 
\preceq \gamma_o \left( \gamma_o I_h + \phi(\vx_1) \phi(\vx_1)^T \right)^{-1} = R^{(1)},
\end{align}
meaning that the operator norm of $R^{(b)}$ under arbitrary batch size $b \geq 1$ is upper bounded by the norm of $R^{(1)}$, which corresponds to the single sample update. 
Under $b=1$, from Sherman-Morrison formula, 
\begin{align}
    \Ebb_{x_1} \left[ \gamma_o \left( \gamma_o I_h + \phi(\vx_1) \phi(\vx_1)^T \right)^{-1} \right]
    =\Ebb_{x_1} \left[ I_h - \frac{\phi(\vx_1) \phi(\vx_1)^T}{\gamma_o + \llv \phi(\vx_1) \rrv^2} \right]
    = I_h - \frac{\Sigma}{ \gamma_o + 1}. 
\end{align}
In the last line we used the assumptions, $\Ebb_x[\phi(\vx) \phi(\vx)^T] = \Sigma$ and $\llv \phi(\vx) \rrv^2 = 1$. 
Thus, for any $b \geq 1$, the largest eigenvalue $\lambda_{\max} \left( E_{X} [ (R^{(b)})^2 ] \right)$ is bounded by 
\begin{align}
\lambda_{\max} \left( E_{X} [ (R^{(b)})^2 ] \right)
\leq \lambda_{\max} \left( E_{X} [ R^{(b)} ] \right)
\leq \lambda_{\max} \left( E_{\vx_1} [ R^{(1)} ] \right)
\leq \lambda_{\max} \left( I_h - \frac{\Sigma}{ \gamma_o + 1} \right)
\leq 1 - \frac{\rho}{1 + \gamma_o}.
\end{align}
\end{proof}

\begin{proposition} \label{prop_iter_tc_bound_iid}
    Under the Assumptions \ref{subsec_iid_bound}, the expectation of the norm of offline-online gap over i.i.d. samples satisfies
    \begin{align}
        \Ebb_{X_{1:t}} \left[ \llv g_t (\cdot) \rrv^2_{\mcH} \right]
        \leq \frac{ 8 b (1 + \gamma)^2 (1 + \gamma_o)^2 \llv \vw_* \rrv^2 }{\gamma^2 \rho^2 } 
    \left( 
    \frac{2}{b t} + \left( 1 - \frac{\rho}{1+\gamma_o} \right)^{t/2}
    \right).
    \end{align}
\end{proposition}

\begin{proof} 
The offline-online gap $g_t (\cdot)$ follows
\begin{align}
    g_t (\cdot)
    &= f_{on} (\cdot; X_t, Z_t) - f_{\off} (\cdot ; X_t, Y_t)
    \nonumber \\
    &= \left[ g_{t-1} (\cdot) - g_{t-1} (X_t) (\gamma_o I + K_{tt})^{-1} k(X_t, \cdot) \right] 
    - E_t Q_t^{-1} \left[ r_t (\cdot) - r_t(X_t) (\gamma_o I + K_{tt})^{-1} k(X_t, \cdot)\right],
\end{align}
where $K_{tt} = k(X_t, X_t)$ is a $b \times b$ gram matrix. 
The first term is rewritten as
\begin{align}
  \left[ g_{t-1} (\cdot) - g_{t-1} (X_t) (\gamma_o I + K_{tt})^{-1} k(X_t, \cdot) \right] 
  &= \left[ I_h - \phi(X_t) \left(\gamma_o I_b + \phi(X_t)^T \phi(X_t) \right)^{-1} \phi(X_t)^T \right] \bm{g}_{t-1} 
  \nonumber \\
  &= \left[ I_h - \left(\gamma_o I_b + \phi(X_t) \phi(X_t)^T \right)^{-1} \phi(X_t) \phi(X_t)^T \right] \bm{g}_{t-1} 
  \nonumber \\
  &= \gamma_o \left(\gamma_o I_h + \phi(X_t) \phi(X_t)^T \right)^{-1} \bm{g}_{t-1},
\end{align}
where vector $\bm{g}_{t-1}$ was defined by $g_{t-1} (\vx) = \bm{g}_{t-1}^T \phi(\vx)$. 
We used push-through identity in the second line. 
For brevity, let us denote the operator term in the equation above and the second term of the online-offline gap equation by
\begin{subequations}
\begin{align}
    R_t &\equiv \gamma_o \left(\gamma_o I_h + \phi(X_t) \phi(X_t)^T \right)^{-1}, \label{eq_def_R_t}
    \\
    h_t (\cdot) &\equiv - E_t Q_t^{-1} \left[ r_t (\cdot) - r_t(X_t) (\gamma_o I + K_{tt})^{-1} k(X_t, \cdot)\right],
\end{align}
\end{subequations}
Here, $R_t$ is $R^{(b)}$ defined with $t$-th sample set $X_t$ (we omitted the superscript for brevity). 
Then, applying it recursively, we have
\begin{align}
    \bm{g}_t 
    &= R_t \bm{g}_{t-1} + \bm{h}_{t} 
    \nonumber \\
    &= R_t R_{t-1} \bm{g}_{t-2}  + R_t \bm{h}_{t-1}  + \bm{h}_{t}  
    = ... = \sum_{s=0}^{t-1} \left( \prod_{m=0}^{s-1} R_{t-m} \right) \bm{h}_{t-s}.
\end{align}
Here, we defined the matrix product in a descending order:
$\prod_{m=0}^{s-1} R_{t-m} \equiv R_{t} R_{t-1} .. R_{t-(s-1)}$.
Therefore, 
\begin{align}
    \left( \Ebb_{X_{1:t}} \left[ \llv g_t (\cdot) \rrv^2_{\mcH} \right] \right)^{1/2}
    &= \left( \Ebb_{X_{1:t}} \left[ \llv \sum_{s=0}^{t-1} \left( \prod_{m=0}^{s-1} R_{t-m} \right) h_{t-s} (\cdot) \rrv_{\mcH}^2 \right] \right)^{1/2}
    \nonumber \\
    &\leq \sum_{s=0}^{t-1} \left( \Ebb_{X_{1:t}} \left[ \llv \left( \prod_{m=0}^{s-1} R_{t-m} \right) h_{t-s} (\cdot) \rrv_{\mcH}^2 \right] \right)^{1/2}.
\end{align}
The term inside the summation is further written as
\begin{align}
    &\Ebb_{X_{1:t}} \left[ 
    \llv  \left( \prod_{m=0}^{s-1} R_{t-m} \right) \bm{h}_{t-s} \rrv_{\mcH}^2 \right]
    \nonumber \\
    &= \Ebb_{X_{1:t-s}} \left[ \bm{h}_{t-s}^T \Ebb_{x_{t-s+1}} \left[ R_{t-s+1}^T ... R_{t-1}^T \Ebb_{x_t} [R_t^T R_t] R_{t-1} ... R_{t-s+1} \right] \bm{h}_{t-s} \right]
    \nonumber \\
    &\leq \lambda_{\max} \left( \Ebb_{x_t} [R_t^T R_t] \right)\Ebb_{X_{1:t-s}} \left[ \bm{h}_{t-s}^T \Ebb_{x_{t-s+1}} \left[ R_{t-s+1}^T ... \Ebb_{x_{t-1}} [R_{t-1}^T R_{t-1}] ... R_{t-s+1} \right] \bm{h}_{t-s} \right]
    \nonumber \\
    &\leq \left( \prod_{m=0}^{s-1} \lambda_{\max} \left( \Ebb_{x_{t-m}} [R_{t-m}^2]\right) \right) \Ebb_{X_{1:t-s}} \left[ \llv h_{t-s} (\cdot) \rrv_{\mcH}^2 \right]. 
\end{align}
The second line holds because $R_t$ only depends on the instantaneous sample $x_t$ and $h_{t-s}$ is a function of all samples up to $t-s$.

The former term, $\lambda_{\max} \left( \Ebb_{x_{t-m}} [R_{t-m}^2]\right)$, is upper bounded by Lemma \ref{lemma_Rb_shrinkage}, while regarding the  latter term, from Eq. \ref{eq_En_Qinv_rn_bound} in Appendix \ref{subsec_noniid_bound}, 
$\llv h_{t-s} (\cdot) \rrv_{\mcH}^2 \leq \frac{1}{\gamma} \llv E_{t-s} \rrv^2$,
thus, we have
\begin{align}
    \Ebb_{x_{1:t-s}} \left[ \llv h_{t-s} (\cdot)\rrv_{\mcH}^2 \right]
    \leq \frac{1}{\gamma} \Ebb_{X_{1:t-s}} \left[ \llv E_{t-s} \rrv^2 \right],
\end{align}
which can be upper bounded by Lemma \ref{lemma_Ebb_Et_bound}. 
Therefore, we have
\begin{align}
    \Ebb_{X_{1:t}} \left[ 
    \llv  \left( \prod_{m=0}^{s-1} R_{t-m} \right) h_{t-s} (\cdot)\rrv_{\mcH}^2 \right]
    \leq \left( 1 - \frac{\rho}{1 + \gamma_o} \right)^s \frac{b (1+\gamma)^2 \llv \vw_*\rrv^2}{\gamma^2 (b(t-s-1)+1)},
\end{align}
and thus
\begin{align}
    \left( \Ebb_{X_{1:t}} \left[ 
    \llv g_t (\cdot) \rrv_{\mcH}^2 \right] \right)^{1/2}
    \leq \sum_{s=0}^{t-1} \left( 1 - \frac{\rho}{1 + \gamma_o} \right)^{s/2} \frac{\sqrt{b}(1+\gamma) \llv \vw_*\rrv}{\gamma \sqrt{b(t-s-1)+1}}
\end{align}
Let us introduce $\alpha \equiv \rho/(1+\gamma_o)$ for brevity, and split the summation into two parts by $s < t/2$ and  $s \geq t/2$.
Excluding the constant term, the first half follows
\begin{align}
    \sum_{s=0}^{t/2-1} \frac{\left( 1 - \alpha \right)^{s/2}}{\sqrt{b(t-s-1) + 1}}
    \leq 
    \frac{1}{\sqrt{bt/2 + 1}} \sum_{s=0}^{t/2-1} \left( 1 - \alpha \right)^{s/2}
    = \frac{1}{ \sqrt{bt/2 + 1} } \frac{1 - (1 - \alpha)^{t/4}}{1 - \sqrt{1 - \alpha}}
    \leq \sqrt{\frac{2}{bt}} \frac{2}{\alpha}.
\end{align}
In the last inequality, we used $1 - \sqrt{1-\alpha} \geq \alpha/2$. 
Similarly, the latter half follows
\begin{align}
    \sum_{s=t/2}^{t-1} \frac{\left( 1 - \alpha \right)^{s/2}}{\sqrt{b(t-s-1) + 1}}
    \leq (1 - \alpha)^{t/4} \sum_{s=0}^{t/2-1} (1 - \alpha)^{s/2}
    = (1 - \alpha)^{t/4} \frac{1 - (1-\alpha)^{t/4}}{1 - \sqrt{1-\alpha}}
    \leq \frac{2}{\alpha} (1 - \alpha)^{t/4}.
\end{align}
By combining the two terms, 
\begin{align}
    \left( \Ebb_{X_{1:t}} \left[ 
    \llv g_t (\cdot) \rrv_{\mcH}^2 \right] \right)^{1/2}
    \leq \frac{ 2\sqrt{b} (1 + \gamma) \llv \vw_* \rrv}{\gamma \alpha } \left( 
    \sqrt{ \frac{2}{b t}} + (1 - \alpha)^{t/4}
    \right).
\end{align}
Using $(a + b)^2 \leq 2 (a^2 + b^2)$, we have
\begin{align}
    \Ebb_{X_{1:t}} \left[ 
    \llv g_t (\cdot) \rrv_{\mcH}^2 \right] 
    \leq \frac{ 8 b (1 + \gamma)^2 \llv \vw_* \rrv^2 }{\gamma^2 \alpha^2 } 
    \left( 
    \frac{2}{b t} + (1 - \alpha)^{t/2}
    \right).
\end{align}
Inserting $\alpha = \rho/(1+\gamma_o)$, we obtain the proposition. 
\end{proof}

Because $0 < \alpha < 1$, at a large $t$ limit, the error converges with $\Order(\tfrac{1}{t})$ rate. $b > 1$ does not accelerate the convergence here because the shrinkage factor under arbitrary $b$ is bounded by $b=1$ case in Lemma \ref{lemma_Rb_shrinkage}. 

\section{Additional numerical results}

\subsection{Application to a Standard Image Recognition Task}
Here we ask whether the same benefit we observed in the continual learning scenario appears in a standard randomly ordered CIFAR-10 image-recognition task. In contrast to the continual-learning settings, corrected targets did not clearly improve single-pass mini-batch SGD. However, an advantage emerged under multi-epoch training (Fig. \ref{fig:conv_sgd_supp_var}B), suggesting that target correction is most beneficial in the presence of structured distributional shifts.

\subsection{Comparison with Soft-Labeling}
Soft-labeling \citep{szegedy2016rethinking, muller2019does}, where a smoothed label is used instead of the one-hot target label, can be interpreted as a form of target correction.  
To dissociate the potential impact of this soft-label effect in the iterative target correction algorithm, we also implemented SGD with soft labeling on the two datasets we tested. As depicted in Table \ref{tab:replay_comparison}, the performance gain from soft labeling over SGD was around 1\% in CIFAR-10, and slightly negative in CORe50, indicating that soft labeling doesn't explain the gain observed by the target correction in non-linear networks. 

\begin{table}[th]
\caption{
Final test accuracy (\%, mean $\pm$ SEM over 10 seeds on CIFAR-10 and 5 seeds on CORe50).
Memory-based methods use $M=0.1N_{\mathrm{train}}$ (CIFAR-10: $M=3000$; CORe50: $M=600$) and replay batch $m=4$.
In TC and DER+TC, NTK was updated at the beginning of each task.
}
\label{tab:replay_comparison}
\begin{center}
\begin{tabular}{lcc}
\multicolumn{1}{c}{\textbf{METHOD}}
&
\multicolumn{1}{c}{\textbf{CIFAR-10 / DOMAIN-IL}}
&
\multicolumn{1}{c}{\textbf{CORe50 / CLASS-IL}}
\\ \hline \\
Vanilla SGD & $57.19 \pm 2.61$ & $13.59 \pm 1.27$ \\
Soft Label (SL) & $58.05 \pm 2.45$ & $11.23 \pm 0.74$ \\
ER & $69.49 \pm 1.36$ & $28.29 \pm 1.40$ \\
DER & $66.64 \pm 1.32$ & $12.04 \pm 0.60$ \\
DER++ & $\mathbf{69.90 \pm 1.45}$ & $27.59 \pm 1.61$ \\
TC (with 10\% samples)  & $64.57 \pm 1.92$ & $22.54 \pm 1.09$ \\
DER+TC & $67.64 \pm 1.44$ & $\mathbf{32.04 \pm 1.21}$ \\
\end{tabular}
\end{center}
\end{table}

\subsection{Dark Experience Replay with Target Correction}
\label{sec:replay_baseline_comparison}

In Dark Experience Replay (DER), previous network outputs, rather than the ground-truth targets, are replayed \citep{buzzega2020dark}, suggesting that non-ground-truth targets can be beneficial. This motivated us to combine iterative target correction (TC) with DER and ask whether corrected targets provide better replay targets than earlier network outputs.

We implemented this variation by applying target correction to each sample newly added to the memory buffer for DER replay, using the samples already stored in the buffer. We then used the corrected target $z^c$ as the replay target instead of the prediction $\hat{y}$ obtained when the sample was added to the memory buffer (see Appendix \ref{subsec:replay_protocol}).

Table \ref{tab:replay_comparison} compares the performance of this DER + TC method with vanilla DER and several other baselines. DER + TC improves prediction accuracy over DER on both the domain-incremental CIFAR-10 task and the class-incremental CORe50 task. On the latter dataset, it also outperforms DER++ \citep{boschini2022class}. These results suggest that the target correction framework is potentially useful for modifying targets in replay buffers, not just for new samples in online learning.




\section{Implementation details}
Below we provide numerical experiment details. Please also see the corresponding repository \url{https://github.com/Hiratani-Lab/target_correction} for further details. 

\subsection{Toy model (Figs. \ref{fig:kernel_reg} and \ref{fig:toy_online})} \label{subsec_toy_model_det}
We generated the one-dimensional data, used for illustration of effective target shift and its correction, using a Gaussian Process (GP) with radial basis function (RBF) kernel. 
Specifically, we defined $X_{grid}$ and $X_{grid} = [0.0, 0.005,...,1.0]$ and sampled the true curve $f_*(x)$ by $f_*(X
_{grid}) \sim \gauss (0, k_{\text{rbf}} (X_{grid},X_{grid}))$, where $k_{\text{rbf}} (x,x') \equiv \exp(-(x-x')^2 / \sigma_{\text{rbf}}^2)$ with $\sigma_{\text{rbf}}^2 = 0.1$. 
We then sub-sampled mutually exclusive training and test sets $X_{\text{train}}$ and $X_{\text{test}}$ randomly from $X_{grid}$, with $|X_{\text{train}}| = 40$ and $|X_{\text{test}}| = 160$. 
The targets $Y_{\text{train}}$ and $Y_{\text{test}}$ were generated by adding Gaussian noise to $f_* (x)$. 
For $x_i \in X_{\text{train}}$, the corresponding target is $y_i = f_* (x) + \sigma_{tn} \xi$, where $\sigma_{tn} = 0.3$ and $\xi$ is a random variable sampled independently from a Gaussian distribution $\gauss(0,1)$. 

In Fig. \ref{fig:kernel_reg}, we used the same RBF kernel with the generative model for offline and online learning model, meaning that the kernel is defined by $k(x,x') = \exp(-(x-x')^2 / \sigma_{\text{rbf}}^2)$. 
We set the regularizer amplitude of the offline learning $\gamma$ to $\gamma = 1.0$, and the learning rate of online learning $\eta$ to $\eta = 0.5$. We picked the latter, $\eta$, to make the diagonal terms of $\gamma I + K(X,X)$ and $\tfrac{1}{\eta} I + K^U (X,X)$ the same (i.e., $\gamma + 1 = 1/\eta$). 

In Fig. \ref{fig:toy_online}, the data was generated in the same way as Fig. \ref{fig:kernel_reg}, but we instead used a finite-size kernel defined by $k(x,x') = \tanh(Jx)^T \tanh(Jx')$ to illustrate the exact match between iterative online and mini-batch updates, and their offline formulations (Eqs. \ref{eq_okr_closed_zero_reg} and \ref{eq_minibatch_predictor}, respectively). We set $J$ to be a 100-dimensional vector, where each element was sampled independently from a Gaussian distribution $\gauss(0,1)$. 

\begin{figure}[t]
  \begin{center}
    \centerline{\includegraphics[width=1.0\columnwidth]{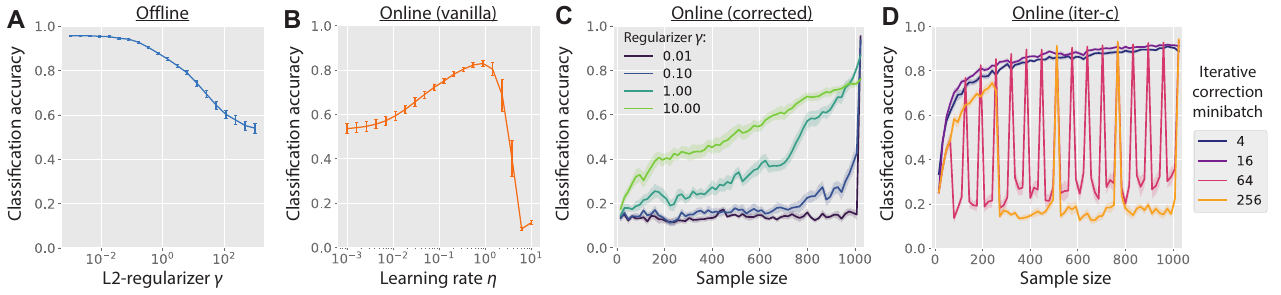}}
    \caption{
    Offline and online regression using NTK applied to MNIST. 
    \textbf{A)} Offline learning performance as a function of the L2-regularizer $\gamma$. 
    \textbf{B)} Vanilla online learning performance as a function of the fixed learning rate $\eta$. 
    \textbf{C)} Online learning with corrected targets under various regularizers $\gamma$ of the target offline predictor. Note that the online learning process itself was implemented without a regularizer as noted in Theorem \ref{target_correction_theorem}. 
    \textbf{D)} Learning curves of iterative target correction with various minibatch sizes for the iterative correction (Eq. \ref{eq_citer_Znew}) under $\gamma = 0.01$ and $\eta = 1.0$. In all the curves, weight updates were implemented in a fully online, sample-by-sample manner. In panels C and D, the test accuracy was measured after every 16 weight updates. Error bars and shaded areas represent SEMs over 10 random seeds. 
    }
    \label{fig:conv_ntk_supp}
  \end{center}
\end{figure}

\subsection{ Neural Tangent Kernel (NTK) Regression Models (Figs. \ref{fig:conv_ntk} and \ref{fig:conv_ntk_supp})} \label{subsec_ntk_reg_det} 
For NTK regression, we used a NTK of an infinite-width convolutional neural network with two convolutional layers and one dense layer. We set the size of receptive field to (3,3) and used average pooling with (2,2) window and (2,2) strides for both convolutional layer, and set the activation function to ReLU. 
We estimated NTK analytically using Python \verb|neural_tangents| library \citep{neuraltangents2020}. 
The MNIST images were whitened at each pixel over samples, and presented to the predictor. We sub-sampled 1024 training data and 256 test data to focus on sparse sample regime. In the class-incremental task setting depicted in Fig. \ref{fig:conv_ntk}D, we ordered 1024 training data from label 0 to label 9. 

In Figs. \ref{fig:conv_ntk}A-C, we set $\gamma = 0.01$ and $\eta = 1.0$ based on the parameter search in Figs. \ref{fig:conv_ntk_supp}A and B. We set $\gamma_o$ from Eq. \ref{eq_iter_correction_loss} to $\gamma_o = 0$ throughout the numerical experiment because $K_{nn}$ was full rank. 
Similarly, from parameter search, in Fig. \ref{fig:conv_ntk}D, we used $\gamma = 0.001$ for the offline learning (blue line) and $\eta = 0.001$ for the vanilla online learning (orange), and used $(\gamma, \eta) = (0.001, 0.001)$ for the online learning with corrected targets (green). We instead used $(\gamma, \eta) = (1.0, 0.3)$ for the iterative target correction. 
In the iterative target correction lines in Fig. \ref{fig:conv_ntk}C and D, we set the mini-batch size for the iterative target correction to $b = 16$, but weight updates were implemented in a fully online, sample-by-sample manner.
A large regularizer stabilizes the learning curve of the online learning with corrected targets (Eq. \ref{eq_def_target_correction_kr}), but lowers the final accuracy. 
Similarly, as depicted in \ref{fig:conv_ntk_supp}D, when the minibatch for iterative correction is large, the final accuracy improves, but the learning curve becomes less stable.

\subsection{ Application to Non-linear Models (Figs. \ref{fig:conv_sgd}, \ref{fig:conv_sgd_supp_var}, and \ref{fig:conv_sgd_supp_hyper}) } \label{subsec_nonlinear_cl_det}

\begin{figure}[t]
  \begin{center}
  \centerline{\includegraphics[width=0.7\columnwidth]{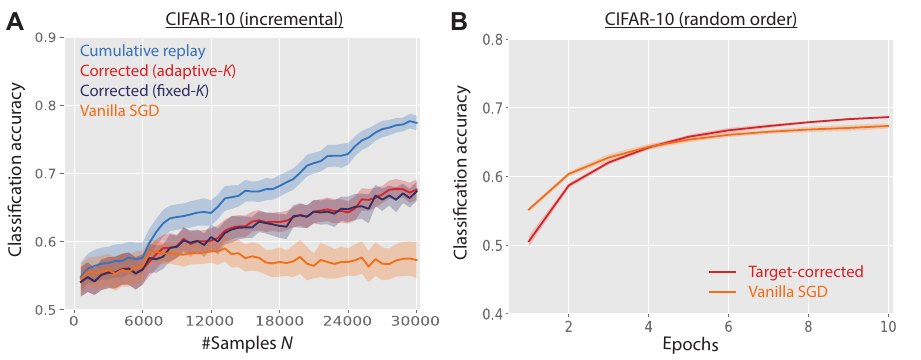}}
    \caption{
    Learning curves under alternative implementations.
    \textbf{(A)} Mini-batch SGD with target correction estimated using a fixed initial NTK (purple line) achieves performance comparable to target correction using an adaptive NTK updated at the beginning of each task (i.e., every $6000$ samples; red line) in the domain-incremental CIFAR-10 task.
    \textbf{(B)} Comparison of the learning curves for target correction and vanilla SGD on the standard CIFAR-10 task with randomly ordered samples (i.e., samples from all classes appear in random order).
    Error bars are SEM across 10 random seeds.
    }
    \label{fig:conv_sgd_supp_var}
  \end{center}
\end{figure}

\begin{figure}[t]
  \begin{center}
 \centerline{\includegraphics[width=1.0\columnwidth]{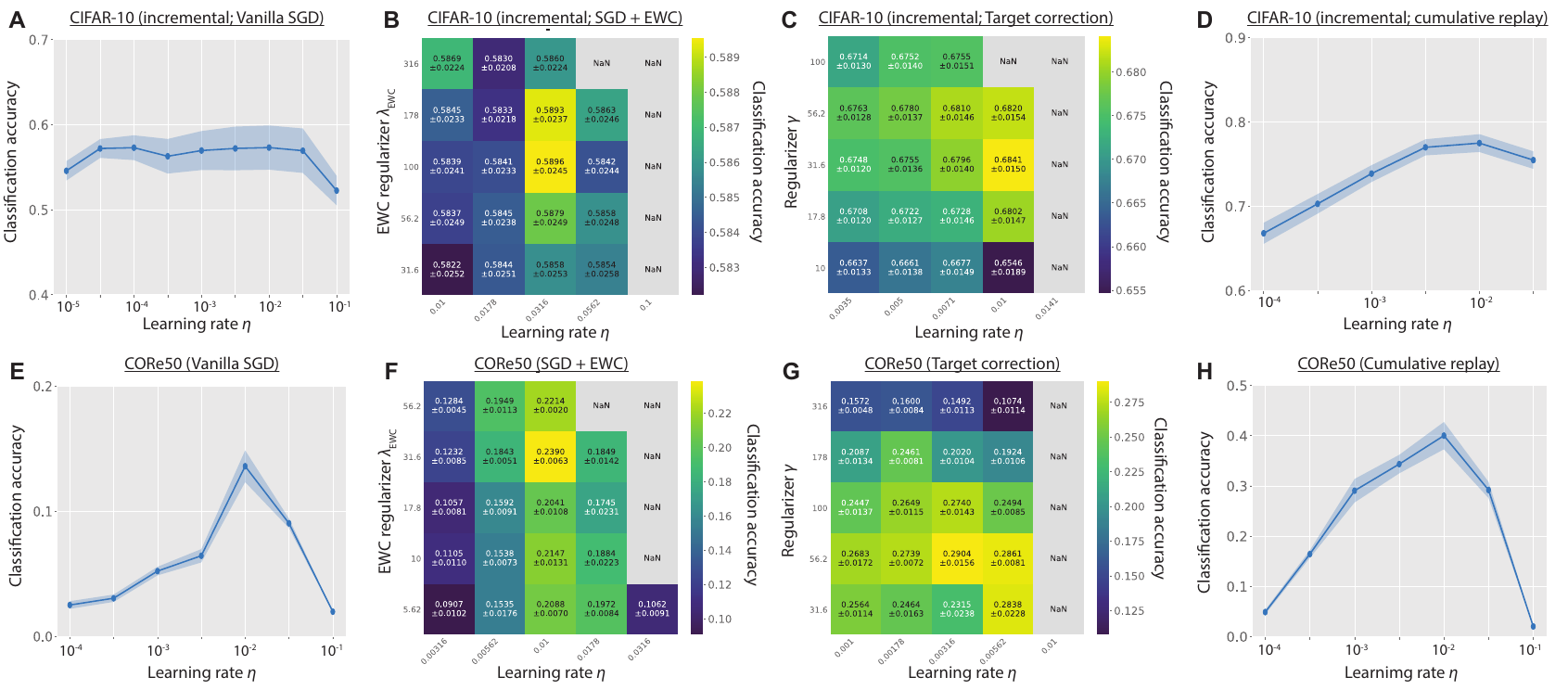}}
    \caption{
    Hyperparameter dependence of nonlinear model training.
    \textbf{(A)} Learning-rate $\eta$ dependence of vanilla SGD performance in the domain-incremental CIFAR-10 task. 
    \textbf{(B-C)} Classification accuracy of EWC (B) and target correction (C) in the domain-incremental CIFAR-10 task under various learning-rate $\eta$ and the regularizer strength $\gamma$
    \textbf{(D)} Performance of SGD with cumulative relpay in the domain-incremental CIFAR-10 task under various learning rate $\eta$
    \textbf{(E-H)} The same as panels A-D, but under the class-incremental CORe50 task. Error bars are SEMs over random seeds.
    }
    \label{fig:conv_sgd_supp_hyper}
  \end{center}
\end{figure}

\subsubsection*{Network Architecture}
We used a small convolutional neural network (CNN) for the image classification experiments on CIFAR-10. The network consisted of two convolutional layers followed by a single fully connected output layer. The first convolutional layer used 32 filters with 3×3 kernels and padding of 1, followed by a ReLU activation and average pooling with a 2×2 window and stride 2. The second convolutional layer used 64 filters with 3×3 kernels and padding of 1, followed by another ReLU activation. Adaptive average pooling was then used to obtain a 7×7 spatial feature map, which was flattened into a 64×7×7-dimensional feature vector. For classification, this feature vector was mapped directly to the output classes using a single linear layer. All convolutional and linear layers were used without bias terms to suppress rank-1 correlations in empirical NTK. Unless otherwise specified, convolutional and linear layers were initialized using the default PyTorch initialization.

For the CORe50 dataset, we used a pre-trained ResNet-18 model as the feature extractor. The parameters of the pre-trained feature extraction layers were kept frozen throughout training. On top of the frozen ResNet-18 features, we added a two-layer fully connected classifier head, consisting of a 100-dimensional hidden layer with a ReLU nonlinearity and a 50-dimensional output layer corresponding to the 50 target classes in CORe50. Only the parameters of these two fully connected layers were trained and updated. Both linear layers were used without bias terms.

\subsubsection*{Datasets}
For Fig.~\ref{fig:conv_sgd}A, the 10 classes in CIFAR-10 were randomly partitioned into five ordered binary classification tasks and presented to the model sequentially. 
We randomly sub-sampled half of the training samples to keep the size of the kernel matrix moderate; as a result, each task consisted of 6,000 samples presented in a single pass with mini-batches. For testing, we used the full \(10^4\) test samples from the CIFAR-10 dataset. 

For Fig.~\ref{fig:conv_sgd}B, we followed the CORe50 task setting used in \citet{lomonaco2017core50}. This is a 50-class class-incremental object recognition setting in which new object classes are introduced sequentially across batches. 
The first batch, or task, contains 10 classes, while each of the remaining eight batches contains 5 classes, resulting in 50 classes in total. Following \citet{lomonaco2017core50}, the classes were selected using a biased policy that promotes broad categorical representation by spreading objects from the same category across different batches. 
The only difference from the original setting is that, instead of using video frames sampled at 20 frames per second, we used 1 frame per second to reduce computational cost.

For Fig.~\ref{fig:conv_sgd_supp_var}B, we evaluated the standard 10-class image classification on CIFAR-10 with a random sample order. Here, instead of strictly online single-pass learning, we considered a multi-pass learning process, since a single pass was not sufficient to achieve decent generalization performance. 
The order of training samples was randomly shuffled once at the beginning of training and then repeated across epochs.

\subsubsection*{Learning Algorithms}
We trained the networks by minimizing the mean-squared error (MSE) loss with stochastic gradient descent (SGD) without momentum using various regularizers and corrected targets as specified below. In all update rules below, we set the mini-batch size of the weight update to 4. 

\paragraph{Vanilla SGD}
The vanilla SGD baseline was trained sequentially on samples in the prescribed order, without target correction, replay, or additional regularization. We used $\eta = 0.01$ in Figs.~\ref{fig:conv_sgd}A and B, and $\eta = 0.03$ in Fig.~\ref{fig:conv_sgd_supp_var}B. See Figs.~\ref{fig:conv_sgd_supp_hyper}A and E for the learning rate dependence.  

\paragraph{SGD with iterative target correction}
For the iterative target correction method, corrected targets were computed using Eq.~\ref{eq_citer_Znew}, but with a block-wise upper-triangular matrix $K^{bU}$ instead of the full upper-triangular matrix $K^U$, to reflect the mini-batch SGD update. 
We set the size of iterative target correction to $b=20$, and set $\gamma_o=0$. We coarsely optimized $\eta$ and $\gamma$ numerically for each task setting. 
The block size of $K^{bU}$ was set to 4, matching the mini-batch size used in SGD. The corrected targets were then used in place of the original one-hot labels when computing the MSE loss and the corresponding gradient. Note that the parameter $\eta$ in Eq.~\ref{eq_def_Con_Coff}, which is used in Eq.~\ref{eq_citer_Znew}, denotes the learning rate used for the SGD weight update.

The empirical NTK was computed with respect to the trainable parameters of the current network. For CORe50, where the pre-trained ResNet-18 feature extractor was frozen, the empirical NTK was computed only with respect to the trainable fully connected classifier head. The specific schedule for updating the empirical NTK depended on the experimental setting.

For Fig.~\ref{fig:conv_sgd}A, corresponding to the CIFAR-10 domain-incremental binary classification setting, the empirical NTK was updated at the beginning of each binary task, i.e., every 6,000 training samples. 
We used \(\eta = 0.01\) and \(\gamma = 31.6\); see Figs.~\ref{fig:conv_sgd_supp_hyper}C. 
In the fixed-$K$ variant, shown in Fig.~\ref{fig:conv_sgd_supp_var}A, the empirical NTK was computed only once using the initial network parameters, and the same corrected targets were reused across all epochs. 

For Fig.~\ref{fig:conv_sgd}B, corresponding to the CORe50 class-incremental setting, each encountered batch was trained for 10 epochs. After each epoch within an encountered batch, we recomputed the empirical NTK using the current model parameters and recalculated the corrected targets. For later encountered batches, target correction was computed based on data from previously encountered batches together with the current batch, following the sequential structure of the CORe50 setting. The hyperparameters were set to \(\eta = 0.00316\) and \(\gamma = 56.2\); see Figs.~\ref{fig:conv_sgd_supp_hyper}G. 

For Fig.~\ref{fig:conv_sgd_supp_var}B, corresponding to the standard CIFAR-10 10-class classification with a fixed shuffled order, we updated the kernel at the beginning of each epoch, as we used multi-epoch training. We used a learning rate of \(\eta = 0.01\) and regularizer \(\gamma = 30.0\).

\paragraph{Elastic Weight Consolidation (EWC) \citep{kirkpatrick2017overcoming}}
For the SGD+EWC baseline, we added a quadratic regularization term to the MSE loss to penalize changes in parameters that were important for previously learned tasks or encountered batches. After each task or encountered batch \(\tau\), we stored the trainable parameters \(\vw^{\ast,(\tau)}\) and estimated the parameter importance \(F_i^{(\tau)}\) using a diagonal approximation, where each diagonal entry was computed from the squared gradient of the loss with respect to parameter \(\vw_i\). During training on the current task or encountered batch \(t\), the objective was
\begin{align}
\mathcal{L}_{\mathrm{EWC}}^{(t)}(\vw)
= \mathcal{L}_{\mathrm{MSE}}^{(t)}(\vw)
+ \frac{\lambda_{\mathrm{EWC}}}{2}
\sum_{\tau < t}
\sum_i
F_i^{(\tau)}
\left(\vw_i-\vw_i^{\ast,(\tau)}\right)^2 .
\end{align}
where \(\vw_i\) denotes the current parameter, \(\vw_i^\ast\) denotes the stored parameter value after previous training, and \(F_i\) is the estimated importance weight. We used $(\eta, \lambda_{\text{EWC}}) = (0.00316, 100)$ for Figs. \ref{fig:conv_sgd}A and $(0.01, 31.6)$ fo  Figs. \ref{fig:conv_sgd}B (see Figs. \ref{fig:conv_sgd_supp_hyper}B and F for parameter dependence).

\paragraph{Cumulative Replay}
For the cumulative replay baseline, the model was trained using all samples encountered so far. That is, after a new task or encountered batch was introduced, the training set was expanded to include both the new data and all previously observed data. This baseline therefore provides a reference for performance when previous data remain available for replay. In both continual CIFAR-10 and CORe50 tasks, we set the learning rate $\eta = 0.01$.

\paragraph{Soft Label (SL)}
In the soft label method implemented in Table \ref{tab:replay_comparison}, one-hot labels are replaced by softer targets. The one-hot target $y\in\{0,1\}^{C}$ of each sample is replaced by
\begin{equation}
\tilde{y}=(1-\varepsilon)\,y+\frac{\varepsilon}{C}\mathbf{1},
\end{equation}
where $C$ is the number of output classes and $\varepsilon\in[0,1]$ is the smoothing strength.
Training minimizes the MSE between the network outputs and $\tilde{y}$.
When $\varepsilon=0$, $\tilde{y}$ reduces to the original one-hot label.

\subsubsection*{Memory-based Methods}
\label{subsec:replay_protocol}

All memory-based methods use the same sample-count budget $M=0.1N_{\mathrm{train}}$.
For ER, DER, DER++, and DER + TC this is the replay-buffer capacity; for sparse TC it is the number of stored anchors used to construct corrected targets. 

Let $(x_t,y_t)_{t\ge 1}$ be the stream and $\mathcal M$ a buffer of capacity $M$. After seeing the $t$-th sample, reservoir sampling updates $\mathcal M$ by
\begin{equation}
\label{eq:reservoir_update}
\mathcal M
\leftarrow
\begin{cases}
\mathcal M\cup\{e_t\},
&
\text{if }t\le M,
\\[0.35em]
\mathcal M\setminus\{e\}\cup\{e_t\}
\text{ with probability }M/t,
&
\text{if }t>M,
\end{cases}
\end{equation}
where $e$ (the sample dropped from the buffer) is drawn uniformly from the current buffer and $e_t$ is the samples stored for sample $t$.
If $t>M$ and the sample is not accepted, $\mathcal M$ is left unchanged.
This keeps $\mathcal M$ an approximately uniform subset of the examplesseen so far.

On multi-epoch streams (CORe50), the same rule is applied to
\emph{unique} samples only: each original index is considered at most once, on its first occurrence. Repeated views in later epochs neither re-enter the buffer nor increment $t$.

Matching $M$ matches the number of stored inputs, not bytes: DER-style methods additionally store a teaching vector per slot.

\paragraph{ER, DER, and DER++}
ER stores $(x,y)$ and replays hard labels.
DER \citep{buzzega2020dark} stores the stop-gradient prediction
$\hat{y}=f_{\theta_t}(x_t)|_{\mathrm{stopgrad}}$ at insertion time and
replays that frozen output.
DER++ \citep{buzzega2020dark} stores $(x,y,\hat{y})$ and uses both replay terms:
\begin{subequations}
\begin{align}
\mathcal L_{\mathrm{ER}}
&=
\frac1{|\mathcal B_t|}
\sum_{(x,y)\in\mathcal B_t}
\ell(x,y)
+
\frac{\beta}{|\mathcal B_{\mathcal M}|}
\sum_{(x',y')\in\mathcal B_{\mathcal M}}
\ell(x',y'),
\label{eq:er_replay_loss}
\\
\mathcal L_{\mathrm{DER}}
&=
\frac1{|\mathcal B_t|}
\sum_{(x,y)\in\mathcal B_t}
\ell(x,y)
+
\frac{\alpha}{|\mathcal B_{\mathcal M}|}
\sum_{(x',\hat{y}')\in\mathcal B_{\mathcal M}}
\ell(x',\hat{y}'),
\label{eq:der_replay_loss}
\\
\mathcal L_{\mathrm{DER++}}
&=
\mathcal L_{\mathrm{DER}}
+
\frac{\beta}{|\mathcal B_{\mathcal M}|}
\sum_{(x',y')\in\mathcal B_{\mathcal M}}
\ell(x',y').
\label{eq:derpp_replay_loss}
\end{align}
\end{subequations}
Thus $\beta=0$ in DER++ recovers DER, and $\alpha=0$ recovers ER.

\paragraph{Target Correction (TC) with Sparse Samples}
For Fig.~\ref{fig:conv_sgd}C, we used sparse subset of samples for target correction. For each task already encountered, a fraction of one training epoch is kept as anchors. Within the task currently being learned, correction is applied to the samples of the current epoch in order, so samples earlier in that epoch are also anchors for the samples that follow. A fraction of 0 uses no samples from previous tasks (i.e., vanilla SGD); a fraction of 100\% uses one full epoch from every previous task.
In this setting, the reservoir $\mathcal M$ of pairs $(x,y)$ corresponds to the anchor set for the iterative correction in Eq.~\ref{eq_citer_Znew}.
Admission into the buffer $\mathcal M$ followed Eq. ~\ref{eq:reservoir_update}.

\paragraph{DER + TC}
DER+TC keeps DER's current-batch loss on $y$ and the same reservoir
loop, but the replay target is a corrected label $z^{c}$ rather than a
frozen prediction $\hat{y}$.
When a new block is considered for insertion, $z^{c}$ is computed from
Eq.~\ref{eq_citer_Znew} using the \emph{current} reservoir true labels
$Y_{\mathcal M}$ as the prefix; previously stored $z^{c}$ values are
not reused as KRR targets.
Accepted entries store $(x,y,z^{c})$.
The objective is
\begin{align}
\mathcal L_{\mathrm{DER+TC}}
&=
\frac1{|\mathcal B_t|}
\sum_{(x,y)\in\mathcal B_t}
\ell(x,y)
+
\frac{\alpha}{|\mathcal B_{\mathcal M}|}
\sum_{(x',z^{c\prime})\in\mathcal B_{\mathcal M}}
\ell(x',z^{c\prime}).
\label{eq:der_tc_replay_loss}
\end{align}
DER and DER+TC therefore share the replay mechanism and differ only in the stored teaching signal.

\subsubsection*{Implementation}
The small CNN used in the CIFAR-10 experiments was implemented using torch.nn, whereas the pre-trained ResNet-18 model used for CORe50 was constructed from torchvision.models. The experiments were conducted on NVIDIA H100 GPUs and on additional GPUs. Execution time of the nonlinear models was around a few hours per run. 

\end{document}